\documentclass[11pt]{article}

\usepackage{fullpage,times,url,bm}
\usepackage{amsthm,amsfonts,amsmath,amssymb,mathtools,epsfig,color,float,graphicx,verbatim}
\usepackage{algorithm,algorithmic}
\usepackage{bbm}
\usepackage[square,numbers]{natbib}
\usepackage{bbm}
\usepackage{epstopdf}
\usepackage{authblk}
\usepackage{tikz}
\usetikzlibrary{shapes.geometric}
\usetikzlibrary{positioning}
\usepackage{pgfplots}
\usepackage{booktabs}
\usepackage{caption}
\usepackage{subcaption}

\usepackage{hyperref}
\hypersetup{
	colorlinks   = true, 
	urlcolor     = blue, 
	linkcolor    = blue, 
	citecolor   = black 
}

\newtheorem{theorem}{Theorem}[section]
\newtheorem{proposition}[theorem]{Proposition}
\newtheorem{lemma}[theorem]{Lemma}

\newtheorem{fact}[theorem]{Fact}

\newcommand{\reals}{\mathbb{R}}
\newcommand{\E}{\mathbb{E}}

\newcommand{\argmin}[1]{\underset{#1}{\mathrm{argmin}}}

\newcommand{\relu}[1]{\left[ #1 \right]_+}
\newcommand{\set}[1]{\left\{#1\right\}}

\newcommand{\ba}{\mathbf{a}}

\newcommand{\bx}{\mathbf{x}}
\newcommand{\bw}{\mathbf{w}}

\newcommand{\bb}{\mathbf{b}}

\newcommand{\bz}{\mathbf{z}}
\newcommand{\bc}{\mathbf{c}}

\newcommand{\by}{\mathbf{y}}

\newcommand{\Lcal}{\mathcal{L}}
\newcommand{\Ocal}{\mathcal{O}}
\newcommand{\Acal}{\mathcal{A}}
\newcommand{\Bcal}{\mathcal{B}}

\newcommand{\Ccal}{\mathcal{C}}
\newcommand{\Xcal}{\mathcal{X}}

\newcommand{\Dcal}{\mathcal{D}}
\newcommand{\Fcal}{\mathcal{F}}

\newcommand{\Rcal}{\mathcal{R}}
\newcommand{\Ncal}{\mathcal{N}}
\newcommand{\Scal}{\mathcal{S}}

\newcommand{\Ucal}{\mathcal{U}}

\newcommand{\norm}[1]{\|#1\|}
\newcommand{\inner}[1]{\langle#1\rangle}
\newcommand{\p}[1]{\left(#1\right)}
\newcommand{\pcc}[1]{\left[#1\right]}
\newcommand{\abs}[1]{\left|#1\right|}

\newcommand{\pr}{\operatorname*{\mathbb{P}}}
\newcommand{\Exp}{\operatorname*{\mathbb{E}}}

\newtheorem{example}[theorem]{Example}

\newtheorem{assumption}{Assumption}

\newcommand{\secref}[1]{Sec.~\ref{#1}}
\newcommand{\subsecref}[1]{Subsection~\ref{#1}}
\newcommand{\figref}[1]{Fig.~\ref{#1}}
\renewcommand{\eqref}[1]{Eq.~(\ref{#1})}
\newcommand{\lemref}[1]{Lemma~\ref{#1}}

\newcommand{\thmref}[1]{Thm.~\ref{#1}}
\newcommand{\propref}[1]{Proposition~\ref{#1}}
\newcommand{\appref}[1]{Appendix~\ref{#1}}

\newcommand{\asmref}[1]{Assumption~\ref{#1}}
\newcommand{\itemref}[1]{Item~\ref{#1}}

\DeclareMathOperator{\ip}{IP}

\DeclareMathOperator{\AND}{AND}
\DeclareMathOperator{\round}{round}
\DeclareMathOperator{\poly}{poly}
\DeclareMathOperator{\thresh}{thresh}

\makeatletter

\newif\ifanonymize

\anonymizefalse

\def\moverlay{\mathpalette\mov@rlay}
\def\mov@rlay#1#2{\leavevmode\vtop{%
   \baselineskip\z@skip \lineskiplimit-\maxdimen
   \ialign{\hfil$\m@th#1##$\hfil\cr#2\crcr}}}
\newcommand{\charfusion}[3][\mathord]{
    #1{\ifx#1\mathop\vphantom{#2}\fi
        \mathpalette\mov@rlay{#2\cr#3}
      }
    \ifx#1\mathop\expandafter\displaylimits\fi}
\makeatother

\makeatletter
\newcommand{\printfnsymbol}[1]{%
  \textsuperscript{\@fnsymbol{#1}}%
}
\makeatother

\title{Depth Separations in Neural Networks:\\ Separating the Dimension from the Accuracy}

\ifanonymize 
    \author{\textbf{Anonymous Author(s)}}
\else
    \author[1,2]{Itay Safran}
    \author[3]{Daniel Reichman}
    \author[2]{Paul Valiant}
    \affil[1]{Ben-Gurion University}
    \affil[2]{Purdue University}
    \affil[3]{Worcester Polytechnic Institute}
\fi 

\date{}

\begin{document}

\maketitle

\begin{abstract}
    We prove an exponential size separation between depth 2 and depth 3 neural networks (with real inputs), when approximating a $\mathcal{O}(1)$-Lipschitz target function to constant accuracy, with respect to a distribution with support in the unit ball, under the mild assumption that the weights of the depth 2 network are exponentially bounded. This resolves an open problem posed in \citet{safran2019depth}, and proves that the curse of dimensionality manifests itself in depth 2 approximation, even in cases where the target function can be represented efficiently using a depth 3 network. Previously, lower bounds that were used to separate depth 2 from depth 3 networks required that at least one of the Lipschitz constant, target accuracy or (some measure of) the size of the domain of approximation scale \emph{polynomially} with the input dimension, whereas in our result these parameters are fixed to be \emph{constants} independent of the input dimension: our parameters are simultaneously optimal. Our lower bound holds for a wide variety of activation functions, and is based on a novel application of a worst- to average-case random self-reducibility argument, allowing us to leverage depth 2 threshold circuits lower bounds in a new domain. 
\end{abstract}

\section{Introduction}

There is significant empirical evidence suggesting that depth plays a crucial role in the practical success of deep learning~\citep{lecun2015deep,he2016deep}. From a purely theoretical perspective, while depth 2 neural networks are known to be universal approximators for continuous functions over compact domains~\citep{cybenko1989approximation,hornik1989multilayer,leshno1993multilayer}, it is now well-established that greater depth may be necessary to obtain a compact representation of certain target functions \citep{eldan2016power,telgarsky2016benefits,poggio2017and,daniely2017depth,yarotsky2017error,liang2016deep,safran2017depth,safran2019depth,venturi2022depth,safran2022optimization,nichani2023provable,safran2024many}. Perhaps the most stark such examples are when separating depth 2 from depth 3 -- following the seminal work of~\citet{eldan2016power}, many works have demonstrated \emph{depth separations} where depth can be exponentially more beneficial than width, even when increased by just one: There exists a function $f:\reals^d\to\reals$ that can be approximated to accuracy $\varepsilon>0$ using a network of depth 3 and width $\poly(d,1/\varepsilon)$, yet any depth 2 network which approximates $f$ to the same accuracy would require exponentially many more neurons (in $d$ or $1/\varepsilon$).

Despite the growing number of situations where such separations can be shown, known results often involve contrived settings with complicated distributions and oscillatory target functions. On the other hand, functions arising in machine learning settings are often smooth, with bounded fluctuations. Moreover, it was recently shown that in some examples constructed to show depth separations with neural networks, the same property which prevents approximation using the shallow architecture, may also prove detrimental for \emph{learning} the target function, even when using a network which is sufficiently deep to express the target function efficiently \citep{malach2019deeper,malach2021connection}. In light of this, several recent works have shifted their focus to devising depth separation results for functions that are arguably more natural \citep{safran2017depth,safran2019depth,safran2022optimization,ren2023depth,nichani2023provable,safran2024many}. 
The rationale is that such results could be better aligned with learning problems arising in applications, rather than contrived examples that may not be representative of broader classes of functions.

In~\citet{safran2019depth}, the authors observe: given a width 
lower bound for a depth two neural network approximating a target function $f$, that we can trade off the Lipschitz constant of $f$ and the target accuracy parameters  by rescaling $f$ by a multiplicative factor. Moreover, one can also strengthen the accuracy lower bound, at the cost of dilating the domain of approximation, by using a change of variables (see \citep[Theorem~9]{safran2019depth} for a precise statement). This observation suggests, that if either one of these parameters scales with the input dimension $d$ (as is the case in all known separation results), then it is not possible to pinpoint whether the true cause of the difficulty in the approximation stems from the input dimension itself, or from the remaining parameters that were forced to scale with it. In other words, it is not clear if known approximation lower bounds for depth 2 networks are a manifestation of the curse of dimensionality, or if the difficulty lies somewhere else. To study the root cause for the 
inability of polynomially sized depth 2 networks to approximate certain functions, the authors pose the following question:

\begin{quote}
	\emph{Can we show a superpolynomial depth 2 vs.\ depth 3 neural network separation result in terms of the dimension $d$, for approximating $\Ocal(1)$-Lipschitz functions up to constant accuracy $\varepsilon$, on a domain of bounded radius (all independent of $d$)?} \citep{safran2019depth} 
\end{quote}

Their main result, is that in contrast to what previous separation results suggested, if one considers \textit{radial} target functions that are commonly used to show such depth separations, then an exponential dependence on $d$ is not possible. This is demonstrated by providing a general approximation result wherein width $\poly(d)$ suffices.

Motivated by the above question, in this work, we prove a separation result between depth 2 and depth 3, where the target function is $\Ocal(1)$-Lipschitz, the domain of approximation is contained inside the unit ball, and the separation is exponential even if the target accuracy is an absolute constant. (See Table~\ref{tbl} for a comparison). Stated somewhat informally, the following is our main result in this paper:

\begin{theorem}[Informal version of \thmref{thm:main}]\label{thm:informal}
    There exists a sequence of distributions $\left\{\Dcal_{4d}\right\}_{d=1}^{\infty}$ supported in the $4d$-dimensional Euclidean ball, and a sequence of $\Ocal(1)$-Lipschitz functions $f_d:\reals^{4d}\to\reals$, such that, using any $\Ocal(1)$-Lipschitz or threshold activation function $\sigma$ for the neural networks, we have the following:
    \begin{itemize}
    \item $f_d$ is approximable to arbitrary accuracy $\varepsilon$ (in the $L_\infty$ sense) by a depth 3 neural network $\Ncal'_d$ with size and weights bounded by $poly(d,1/\varepsilon)$.
    \item $f_d$ is inapproximable to accuracy better than $0.05$ in the $L_2$ sense by any depth 2 neural network $\Ncal_d$, unless $\Ncal_d$ has either width or weights 
    exponential in the dimension $d$.
	\end{itemize}
\end{theorem}

\captionsetup{font={footnotesize, sf}, labelfont=bf}
\begin{table}
    \caption{\textit{Several known lower bounds for approximation using depth 2 neural networks, where the target function is scaled to be $\Ocal(1)$-Lipschitz, and for the sake of comparison, the distribution is scaled to either contain the unit ball (for compactly supported distributions), or normalized to have most of its probability mass contained within the unit ball (non-compactly supported distributions). An asterisk $^*$ denotes that the result in \citet{daniely2017depth} is not for a radial function, but can be reduced to one taking the form $\bx\mapsto\frac{1}{2\pi d^{3}}\sin(2\pi d^{3}\norm{\bx}^2)$, which oscillates in any direction from the origin (see \citep{safran2019depth,vardi2020neural}). In contrast, our result is non-oscillatory in all such directions, but it can be seen rather as an embedding of a Boolean function into the unit ball. \citet{safran2019depth} use a reduction to the main result of \citet{eldan2016power} to show a lower bound for a non-oscillatory function, but the cost of removing this oscillation is that the accuracy gets worse by a factor of $d^{-2}$. To the best of our knowledge, our lower bound, which is highlighted in bold, is the only constant accuracy lower bound in this setting. Since in many machine learning applications the input dimension is quite large, we have that lower bounds that scale as $1/\poly(d)$ are too permissive, in contrast to our constant lower bound.}}
	\label{tbl}
	\centering
	\begin{tabular}{c|c|c|c}
		\toprule
		Reference & Distribution & Function & Accuracy\\
		\midrule
		\citet{eldan2016power} & Radial, heavy-tailed & Radial, oscillatory & $\Omega\p{d^{-5}}$\\
		\citet{daniely2017depth} & Radial$^*$, compactly supported & Radial$^*$, oscillatory & $\Omega\p{d^{-3}}$ \\
		\citet{safran2019depth} & Radial, heavy-tailed, & Radial, non-oscillatory & $\Omega\p{d^{-7}}$\\
		\citet{venturi2022depth} & Product, heavy-tailed & Product, oscillatory & $\Omega\p{d^{-5}}$\\
		\textbf{This paper} & \textbf{Compactly supported} & \textbf{Non-oscillatory} & $\boldsymbol{\Omega\p{1}}$ \\
		\bottomrule
	\end{tabular}
\end{table}

This result resolves the main open question of~\citet{safran2019depth}, showing that depth 3 neural networks can be  essentially exponentially more efficient than depth 2 neural networks, even for the mildest and most representative setting: where the target function $f$ is $O(1)$-Lipschitz, on a domain of radius 1, and for any commonly used activation function.


Our proof technique is quite different than those commonly used in the literature. While it is common to use some functional analysis tool such as Fourier spectrum analysis \citep{eldan2016power} or spherical harmonics \citep{daniely2017depth}, our lower bound relies on a reduction to threshold circuits (with Boolean inputs), where an exact computation lower bound for the IP mod 2 function is used~\cite{hajnal1993threshold}. The key ingredient in our proof is a reduction from a worst-case complexity, where we are able to construct a network which is capable of classifying \emph{all} the inputs correctly; to an average-case complexity, where intuitively, only a constant portion of the inputs are classified correctly. This construction utilizes the randomization of the input in a manner which preserves its output value, yet induces sufficient randomness so as to result in a concentration of measure which provides a worst-case guarantee. This is achieved by a careful analysis of properties of the probability distribution induced by our randomized self-reduction, and much of the proof of the lower bound is dedicated to this part.

We stress that lower bounds for approximate computation of continuous functions with real inputs, using depth 2 neural networks, do not follow in any straightforward way from lower bounds for depth 2 threshold circuits that \emph{exactly} compute a Boolean function over the Boolean hypercube. 
Indeed it is explicitly stated in \citet{eldan2016power} that 
``lower bounds for Boolean circuits do not readily translate to neural networks of the type used in practice, which are real-valued and express continuous
functions".
Furthermore, our paper focuses on domains of constant diameter such as the unit ball, and this restriction rules out various techniques/constructions that are available for domains of non-constant radius, including the unit hypercube $[0,1]^d$.

We point out that our separation holds for a wide family of activation functions that includes commonly used functions like the ReLU, threshold and sigmoidal activations. Similarly to other lower bounds in the literature that are shown with respect to compactly supported distributions (see related work section below), our result assumes a mild exponential upper bound on the magnitude of the weights of the approximating network. In contrast, lower bounds that do not impose any restrictions on the magnitude of the weights, rely on the technique in \citet{eldan2016power}, and use heavy-tailed data distributions. To the best of our knowledge, it remains a major open problem to show a superpolynomial separation result between depth 2 and depth 3, with respect to a distribution with bounded support, and when allowing unbounded weights.

The remainder of this paper is structured as follows: After presenting our contributions in this paper in more detail below, we discuss related work in the literature. In \secref{sec:settings_main} we present the notation used throughout this paper, followed by a formal construction of our separation setting, the set of assumptions used in our results, and the main result in this paper. \secref{sec:lb} details our depth 2 lower bound, as well as provides a sketch of the proof. \secref{sec:ub} details our depth 3 positive approximation result, as well as provides a concrete example of the construction for the ReLU activation function.

\subsection*{Our contributions}

\begin{itemize}
	\item 
	We prove that under mild assumptions on the activation function $\sigma$ (\asmref{asm:lb}), a depth 2 neural network with $\sigma$ activations cannot approximate a certain $\Ocal(1)$-Lipschitz function $f_d$ well (see \eqref{eq:f_d}). It is shown that there exists a simple distribution supported on the Euclidean unit ball, such that $f_d$ cannot be approximated to better than constant accuracy with respect to this distribution, unless the width of the network or the coefficient sizes scale as $\exp\p{\Omega(d)}$ 
  (\thmref{thm:lb}).
	
    \item 
    As an intermediate step in the derivation of our previous result, using an embedding technique, we show a general reduction from Boolean functions that are hard to approximate on the Boolean hypercube, to a function which is hard to approximate on the unit ball (\thmref{thm:ball_reduction}). This result relies on a Rademacher complexity argument to show an approximation lower bound, and may be of independent interest.
 
    \item 
	We show that under related mild assumptions on the activation function $\sigma$ (\asmref{asm:ub}), a depth 3 neural network with $\sigma$ activations can approximate the same function $f_d$ to arbitrary accuracy in an $L_{\infty}$ sense, using width that scales polynomially with the target accuracy and $d$ (\thmref{thm:ub}).
	
    
    \item
    Combining our lower and upper bounds, we obtain our main result (\thmref{thm:main}), which demonstrates the manifestation of the curse of dimensionality in depth 2 approximation, even if the target function is easy to approximate using depth 3. This is in contrast with previously known results that only imply an exponential lower bound in a high-accuracy regime where the target accuracy scales as $1/\poly(d)$.
	
\end{itemize}

\subsection*{Related work}

\paragraph{Separating depth 2 from depth 3, continuous, $L_2$ lower bound setting.}
For concreteness, let us focus here on the setting of separating depth 2 from depth 3 in a continuous, $L_2$ lower bound setting. The seminal work of \citet{eldan2016power} provided the first exponential $L_2$ lower bound for depth 2 neural networks. The authors construct an approximately radial, oscillatory target function, and use a Fourier spectrum analysis argument to prove their unconstrained lower bound (imposing no bounds on the magnitude of the weights). As an artifact of this technique, a superposition argument is required to guarantee that the target function cannot be approximated with fewer than exponentially many neurons. This, however, also forces the Lipschitz constant and the number of oscillations to scale as $\poly(d)$. \citet{venturi2022depth} adapt the technique of Fourier spectrum analysis to a setting with product target functions and distributions, rather than radial ones. Due to this technique, both works \citep{eldan2016power,venturi2022depth} require the target accuracy to scale as $\Omega\p{d^{-5}}$ for the curse of dimensionality to manifest. \citet{safran2024many} also use Fourier spectrum analysis to show a depth separation result between depth 2 and depth 3, for approximating the non-oscillatory $1$-Lipschitz maximum function on the domain $[0,1]^d$, with respect to the uniform distribution. However, their separation is only polynomial (showing that the size of a depth 2 network approximating the target function is polynomially larger than a depth 3 network approximating the function), and for accuracy which scales as $1/\poly(d)$ rather than a constant.

\citet{daniely2017depth} used a simple and elegant harmonic analysis technique, to show a depth 2 vs.\ depth 3 separation, assuming the weights of the depth 2 network are at most exponential in the input dimension, for oscillatory target functions, over a product distribution on two spheres (that can be reduced to a radial distribution---see \citep{safran2019depth,vardi2020neural} for more details). With some careful analysis, it can be shown that with this technique, the approximation error lower bound on a $\Ocal(1)$-Lipschitz function in a domain of constant radius is $\Omega\p{d^{-3}}$. The works~\citep{safran2022optimization,nichani2023provable} prove depth 2 approximation lower bounds for non-oscillatory target functions, by using the main result in \citet{daniely2017depth}, and decoupling the dependence of the input dimension from the accuracy in the lower bound. However, both results require the accuracy to scale as $d^{-3}$ for the curse of dimensionality to manifest, a property inherited by the proof technique being used.

\citet{safran2017depth} and \citet{safran2019depth} use reductions to the main result of \citet{eldan2016power} to derive exponential depth 2 lower bounds for non-oscillatory functions. \citet{safran2017depth} show this for non-Lipschitz ellipsoid indicator functions, and \citet{safran2019depth} show this for a $1$-Lipschitz radial function. Like the previous results using Fourier spectrum analysis, these results require accuracy scaling as $\Omega\p{d^{-5}}$ to have exponential dependence on the input dimension, but additionally, simplifying the function to be non-oscillatory results in an additional $d^{-2}$ factor, for an overall accuracy lower bound of $\Omega\p{d^{-7}}$. As pointed out above, rescaling these results lets us interpret this as a constant approximation accuracy result, but for functions that are $d^7$-Lipschitz. In contrast, our exponential lower bound is simultaneously for a non-oscillatory $O(1)$-Lipschitz function, for constant approximation accuracy.

\paragraph{Connection between neural networks and threshold circuits.} The connection between approximation lower bounds when using neural networks and threshold circuits has been studied in multiple works recently. \citet{martens2013representational} study lower bounds for depth 2 neural networks with non-threshold activation functions, using known threshold circuit lower bounds on the inner product function. Their work differs from ours in a few ways: (i) theirs is a much more direct application of circuit lower bounds to bound \emph{exact} computation of Boolean functions by neural networks, while we show a random self-reduction to bound the hardness of \emph{average-case} approximation; and (ii) we focus on approximation on continuous domains rather than the discrete uniform distribution over the Boolean hypercube; (iii) their setting of ``restricted Boltzmann machines'' is different from our neural network setting. \citet{mukherjee2017lower} derive sub-linear size lower bounds for neural networks by showing reductions to known threshold circuit lower bounds. \citet{vardi2020neural} show barriers for achieving depth separations in neural networks by using reductions to open problems in threshold circuits, and applying known ``natural proof barrier'' results, showing that such separations would solve long-standing open problem that are widely believed to be very difficult. Since their results only apply to networks of depth $4$ or more, they do not apply in the setting investigated in this paper, which to the best of our knowledge does not solve any open problem in circuit complexity. \citet{vardi2021size} use communication complexity to derive size lower bounds for ReLU networks which compute IP mod 2. While this lower bound applies regardless of the depth of the circuit (similarly to the linear lower bound for IP mod 2 for threshold circuits), it is for a network size which is linear (up to logarithmic factors) in the input dimension, whereas we use a similar distribution, but in order to show an exponential separation between depth 2 and depth 3.

\paragraph{Progress on the open question posed in \citet{safran2019depth}.}

To the best of our knowledge, the only works that made progress with the open question posed in \citet{safran2019depth}, are \citep{safran2019depth,hsu2021approximation}. \citet{safran2019depth} observe that the target accuracy, Lipschitz parameter of the target function, and the radius of the domain of approximation, can all be traded off with polynomial factors. As one of their main contributions, the authors show that when all three parameters are held constant, then an exponential separation between depth 2 and depth 3 is not possible for radial functions, by proving a positive $L_{\infty}$ approximation result in this setting where width polynomial in $d$ suffices. This indicates that in order to resolve the question, one must consider non-radial target functions. \citet{hsu2021approximation} consider the question of how many randomly initialized ReLU neurons are required for the approximation of arbitrary functions with a constant Lipschitz parameter, with respect to the \emph{uniform} distribution over $[0,1]^d$, and with high probability. Their main positive approximation result is that perhaps surprisingly, if the target accuracy is fixed, then a depth 2, width $\poly(d)$ random ReLU network will approximate any $\Ocal(1)$-Lipschitz function with high probability (hence -- there exists a deterministic network with this approximation). However, since this result is with respect to a uniform $L_2$ approximation rather than an $L_{\infty}$ approximation, it does not rule out results like ours that are lower bounds over \emph{nonuniform} distributions. \citet{hsu2021approximation} does strongly suggest, however, that results like ours can only be possible for distributions that are far in some sense from the uniform distribution. 
Indeed, our lower bound is for a distribution with support in the unit ball that has most of its probability mass concentrated close to a discrete set of exponentially many points.

\section{Setting and Main Result}\label{sec:settings_main}

In this section, we formally define our setting and present the notation and terminology used throughout the paper, before turning to present our assumptions and main result.

\subsection{Preliminaries and notation}

\paragraph{Notation and terminology.}
We let $[n]$ be shorthand for the set $\{1,\ldots,n\}$. We denote vectors using bold-faced letters (e.g.\ $\bx$) and matrices or random variables using upper-case letters (e.g.\ $X$). Given a vector $\bx=(x_1,\ldots,x_d)\in\reals^d$, we define $\round(\bx)=(\round(x_1),\ldots,\round(x_d))$, where $\round(x)$ rounds $x$ to the nearest integer. We let $\Ucal(A)$ denote the uniform distribution on a set $A\subseteq\reals^d$. We define the Boolean functions $\AND:\{0,1\}^2\to\{0,1\}$, $\AND(x,y)=x\cdot y$; and $\ip_d:\{0,1\}^{2d}\to\{0,1\}$, $\ip_d(\bx,\by)=\inner{\bx,\by}\mod2$. For some vector $\bx_0\in\reals^d$ and real $r>0$, we define the $d$-dimensional ball with radius $r$ centered at $\bx_0$ as $B_{r}^{d}(\bx_0)\coloneqq\{\bx\in\reals^d:\norm{\bx-\bx_0}\le r\}$. For two sets $A,B\subseteq\reals^d$, their \emph{Minkowski sum} is defined as $A+B\coloneqq\{\ba+\bb:\ba\in A,\bb\in B\}$. 

\paragraph{Neural networks and threshold circuits.}

We consider fully connected, feed-forward neural networks, computing functions from $\reals^d$ to $\reals$. A $\sigma$-network consists of layers of neurons. In every layer except for the output neuron, an affine function of the inputs is computed, followed by a computation of the non-linear activation function $\sigma:\reals\to\reals$. The single output neuron simply computes an affine transformation of its inputs. Each layer with a non-linear activation is called a \emph{hidden layer}, and the \emph{depth} of a network is defined as the number of hidden layers plus one. The \emph{width} of a network is defined as the number of neurons in the largest hidden layer, and the \emph{size} of the network is the total number of neurons across all layers. Analogously, we define a \emph{threshold network} as a $\sigma$-network which employs the threshold activation function; namely, where $\sigma(x)=1$ for all $x\ge0.5$, and $\sigma(x)=0$ otherwise. We make the distinction between a threshold network and a \emph{threshold circuit}, where in a threshold circuit, the output neuron also has a non-linear activation rather than a linear activation as is the case for threshold networks.

\subsection{Formal construction}

We begin with defining the distribution used to show our separation result. In short, this distribution is the uniform distribution over a Cartesian product of two sets. The first is a set which is sufficiently spread out so as to guarantee a constant distance between different components of the distribution (which implies that our function is $\Ocal(1)$-Lipschitz). The second can be seen as a continuous embedding of the IP mod 2 function inside the Euclidean unit ball. More formally, let $\bz_1,\ldots,\bz_{2^{2d}}\in B_{0.8}^{2d}(\mathbf{0})$ be a set of points where each pair is at a distance of at least $0.4$ apart (see the full proof for details about its existence), and let $\bx_1,\ldots,\bx_{2^{2d}}$ be an enumeration of the $2d$-dimensional Boolean hypercube. We define the sets
\[
    \Scal_{4d}\coloneqq \set{\p{\bz_i,\frac{1}{4\sqrt{d}}\bx_i}}_{i=1}^{2^{2d}}~ \text{ and } ~\Ccal_{4d}\coloneqq\pcc{0,\frac{1}{12\sqrt{d}}}^{4d}.
\]
Using these sets, we can define the support of our distribution as $\Acal_{4d}\coloneqq\Scal_{4d}+\Ccal_{4d}$, and our distribution to be the uniform distribution over the set $\Acal_{4d}$; namely, the distribution $\Ucal(\Acal_{4d})$. It is interesting to note that the last $2d$ coordinates of this distribution can also be seen as a scaled interpolation between the two distributions $\Ucal([0,1]^d)$ and $\Ucal(\{0,1\}^d)$, where as discussed in the related work subsection, the former is too spread to show constant accuracy lower bounds for $\Ocal(1)$-Lipschitz functions, and the latter is a discrete rather than a continuous distribution, which is the setting where neural networks are typically being used. Given an input $(\bz,\bx,\by)\in\reals^{2d}\times\reals^{d}\times\reals^{d}$, we define our hard to approximate function $f_{d}:B_1^{4d}({\mathbf{0}})\to\reals$ as 
\begin{equation}\label{eq:f_d}
	f_{d}(\bz,\bx,\by)\coloneqq\ip_{d}\p{\round\p{3\sqrt{d}\bx},\round\p{3\sqrt{d}\by}}.
\end{equation}

In words, our function ignores the first $2d$ coordinates, scales the remaining $2d$ coordinates to the unit hypercube, rounds them to the nearest integer, and computes the IP mod 2 function on the resulting input. It is easy to verify that $\Acal_{4d}\subseteq B_{1}^{4d}(\mathbf{0})$ and that $f_d$ is constant on each connected component of $\Acal_{4d}$. Since every pair of such components is at a constant distance apart, we also have that $f_d$ is $\Ocal(1)$-Lipschitz. (See \appref{app:lb_proof} for further detail.) We note that Kirszbraun's theorem implies that $f_d$ can be easily extended to the entire Euclidean ball, with the same $\Ocal(1)$ Lipschitz constant.


\subsection{Assumptions}

Before we can present our main result in this paper, we will first formally state and discuss our assumptions. We begin with formally stating our assumption on the family of activation functions for which our lower bound holds.

\begin{assumption}[Lower bound]\label{asm:lb}
    The activation function $\sigma$ is either the threshold activation or a $\Ocal(1)$-Lipschitz function.
\end{assumption}
Our above assumption is very mild, and it is satisfied by many commonly used activations such as ReLU, threshold and sigmoidal activations. We remark that our result holds in fact for an even wider class of activation functions, and that for the sake of simplicity, we chose to present it here using a function class that is more intuitive to comprehend, yet still encompasses essentially all activation functions that are interesting from a practical perspective. The interested reader is referred to \asmref{asm:lb_relaxed} in the appendix for the most general form of our lower bound assumption.

Having discussed our lower bound assumption, we now move on to state and discuss our upper bound assumption.

\begin{assumption}[Upper bound]\label{asm:ub}
	There exists a constant $c_{\sigma}$ which depends solely on $\sigma$ such that the following holds: For all $R>0$ and any $L$-Lipschitz function $f:[-R,R]\to\reals$, and for any $\delta$, there exist scalars $a,\set{\alpha_i,\beta_i,\gamma_i}_{i=1}^w$, where $w,|a|,|\alpha_i|,|\beta_i|,|\gamma_i|\le c_{\sigma}\frac{RL}{\delta}$ for all $i\in[w]$, such that the function
	\[
	h(x)=a+\sum_{i=1}^w\alpha_i\cdot\sigma(\beta_ix-\gamma_i)
	\]
	satisfies
	\[
	\sup_{x\in[-R,R]}\abs{f(x)-h(x)}\le\delta.
	\]
	
\end{assumption}
The above is a slight modification of Assumption 2 in \citet{eldan2016power}, and is satisfied by many standard activation functions that are used in the literature, which in particular include threshold, ReLU and sigmoidal activations (see \lemref{lem:thresh_approx} in the appendix which implies that the threshold activation satisfies this property, and see Appendix A in \citet{eldan2016power} for a proof for the ReLU activation). We point out that we also require an additional mild requirement in our assumption, that the weights of the approximating network $h$ are bounded in magnitude. This is in order to control the magnitude of the weights in our approximation of $f_d$, and get a valid separation that requires weights of polynomials magnitude, which stands in contrast to our lower bounds, where polynomially bounded weights imply that width exponential in $d$ is necessary.

\subsection{Main result}

We are ready to present our main theorem in this paper:

\begin{theorem}\label{thm:main}
	Consider the sequence of distributions $\left\{\Ucal(\Acal_{4d})\right\}_{d=1}^{\infty}$, and the sequence of $\Ocal(1)$-Lipschitz functions $f_d:\reals^{4d}\to\reals$ defined in \eqref{eq:f_d}. Then for all $C>0$ and sufficiently large $d$, we have for any activation function $\sigma$ that satisfies both \asmref{asm:lb} and \asmref{asm:ub}, that the following hold.
	\begin{itemize}
		\item 
		For any depth 2 $\sigma$-network $\Ncal_d:\reals^d\to\reals$, with weights bounded in magnitude by $C$, we have
		\[
		\Exp_{\bx\sim\Ucal\p{\Acal_{4d}}}\left[\p{f_d(\bx)-\Ncal_d(\bx)}^2\right]>\frac{1}{400},
		\]
		unless $\Ncal_d$ has width at least $\Omega\p{\frac{\exp\p{\Omega\p{d}}}{\poly(C)}}$.
		\item
		For all $\varepsilon>0$, there exists a depth 3, width $\text{poly}(d,1/\varepsilon)$, $\sigma$-network $\Ncal'_d$, such that
		\[
		\sup_{\bx\in\Acal_{4d}}\abs{f_d(\bx)-\Ncal'_d(\bx)}\le\varepsilon,
		\]
	\end{itemize}
	where the asymptotic notation hides constants that depend solely on $\sigma$.
\end{theorem}

The above theorem is an immediate consequence of our lower bound (\thmref{thm:lb}) and upper bound (\thmref{thm:ub}), which will be presented in detail in the following sections.

\section{Lower Bound}\label{sec:lb}

In this section, we present our lower bound for the approximation of the function $f_d$. Thereafter, we provide a proof sketch which conveys the main technical ideas behind the result. We begin with formally stating our lower bound as follows.
\begin{theorem}\label{thm:lb}
	Suppose that $\sigma$ satisfies \asmref{asm:lb}, and let $\Ncal:\reals^{2d}\to\reals$ be a depth 2 $\sigma$-network with weights bounded by $C$ that satisfies
	\begin{equation*}
		\Exp_{\bx\sim \Ucal\p{\Acal_{4d}}}\pcc{\p{\Ncal(\bx) - f_d(\bx)}^2} \le \frac{1}{400}.
	\end{equation*}
	Then, $\Ncal$ has width 
	\[
	    \Omega\p{\frac{\exp\p{\Omega\p{d}}}{\poly(C)}},
	\]
	where the asymptotic notation hides constants that depend solely on $\sigma$.
\end{theorem}

The proof of the above theorem relies on a worst- to average-case reduction in a neural network setting, followed by a reduction to threshold circuits. Given a network which approximates $f_d$ well, we can use it to construct a network which achieves similar accuracy, but with margins that can be made arbitrarily more uniform due to a concentration of measure argument. The crux of our proof is identifying a construction that re-randomizes the input sufficiently well, effectively obtaining a strong enough concentration of measure, while also doing so in a manner which maintains the output value of the function. Thereafter, a very careful technical analysis is required to establish that the accuracy lost due to this re-randomization process is at most a constant. We refer the reader to \subsecref{subsec:sketch} for a more detailed proof sketch of the theorem, and to \appref{app:lb} for the full proof.

We point out that the constant $\frac{1}{400}$ is arbitrary, and our proof technique is capable of improving this to be arbitrarily close to the constant $\frac{1}{16}$ (at the cost of increasing the constants hidden in the asymptotic notation). Since a trivial approximation of a single constant neuron (which returns the value $0.5$) yields accuracy $\frac14$, this indicates that our analysis is very tight, and that adding even exponentially many neurons does not improve upon the trivial approximation by much.

We remark that our weight boundedness assumption is mild, since our lower bound remains exponential in $d$, as long as $C$ is in itself not exponential in $d$. In such a case, where the weights required for expressing $f_d$ must have exponential magnitude, it is known that stable gradient descent must run for exponentially many iterations in order for the weights to reach such a magnitude (see \citet{safran2022optimization} for a more formal result of this kind). This suggests that even if a network of size polynomial in $d$ can approximate $f_d$ well, then in practice, learning such a representation using standard techniques is not tractable, and it is therefore of lesser interest from a practical perspective.

\subsection{Techniques, and proof sketch of \thmref{thm:lb}}\label{subsec:sketch}

In this subsection, we detail the key ideas behind the proof of our lower bound. The reader is referred to \appref{app:lb} for the full proof.

\subsubsection{Step 1: From a continuous distribution on the unit ball to a discrete distribution on the Boolean hypercube}


We begin with assuming that we have a depth 2, $\sigma$-network $\Ncal$, which approximates $f_d$ to accuracy $\frac{1}{400}$, where $\sigma$ satisfies \asmref{asm:lb}, and our goal is to lower bound its width. By our assumption, this network provides a good approximation with respect to the continuous distribution $\Ucal(\Acal_{4d})$ whose support is contained inside the unit ball. Since our aim here is to eventually use lower bounds from threshold circuits to get a lower bound on the width of $\Ncal$, we first reduce this lower bound over the continuous distribution to the discrete distribution $\Ucal(\Scal_{4d})$.
This is done by breaking the distribution to a sum of distributions on the sets $\Scal_{4d}$ and $\Ccal_{4d}$, and using the probabilistic method and Markov's inequality on the randomness induced by the latter continuous component. This implies the existence of a depth 2 neural network $\tilde{\Ncal}$, which has width similar to $\Ncal$, but can approximate $f_d$ on the discrete set $\Scal_{4d}$ to an average square loss of $\frac{1}{399}$. 

Next, we wish to show an approximation lower bound reduction from Boolean functions on the Boolean hypercube to the unit ball. Our reduction technique, which may be of independent interest, relies on embedding a scaled copy of the Boolean hypercube inside the unit ball. Since fitting a hypercube inside the unit ball requires scaling by a factor of $\Theta(\sqrt{d})$, this forces the Lipschitz constant to scale with $d$. To maintain its independence from the input dimension, we first observe that while the unit ball has much smaller volume than the unit hypercube, their packing numbers scale similarly when considering coverings using balls of small (but constant) radii. This is useful, since our distribution has its mass centered at a discrete set of exponentially many points. By embedding the Boolean hypercube into the Cartesian product of a scaled Boolean hypercube and a randomly-chosen and well-spread set which achieves an optimal packing, we are able to maintain Lipschitzness. Lastly, to show that the hardness of approximation property persists, we use a Rademacher complexity argument. Essentially, the hardness in the scaled Boolean hypercube component stems by our assumption that the function is hard to approximate, and the hardness in the packing component stems from its random nature which guarantees that any correlation between the point and its target value is negligible. With this reduction result (see \thmref{thm:ball_reduction} in the appendix), we establish that the existence of $\tilde{\Ncal}$ implies the existence of a $\sigma$-network $\Ncal':\reals^{2d}\to\reals$ which computes $\ip_d$ to an average squared error of $\frac{1}{398}$ on the Boolean hypercube.



\subsubsection{Step 2: From worst- to average-case using randomization}

In the previous step, we constructed a depth 2, $\sigma$-network $\Ncal'$, which approximates $\ip_d$ uniformly over $\{0,1\}^{2d}$, to an average squared error of $\frac{1}{398}$. Intuitively, this means that $\Ncal'$ computes $\ip_d$ well in an average-case sense, since by Markov's inequality, for at least a constant fraction of the inputs $\bx,\by\in\{0,1\}^{d}$, we have that $\round(\Ncal'(\bx,\by)) = \ip_d(\bx,\by)$. Our aim is now to use $\Ncal'$ to construct a depth 2, $\sigma$-network $\Ncal''$, such that $\round(\Ncal''(\bx,\by)) = \ip_d(\bx,\by)$ holds \emph{for all inputs} $\bx,\by\in\{0,1\}^{d}$. To this end, we use a randomization scheme on the input, where we map it to a higher dimensional space, and use a higher dimensional architecture of $\Ncal'$. The purpose of this scheme is to alter the input in a manner which induces as much randomness as possible, but while also keeping its output unchanged. This is achieved by identifying the following three different alterations:
\begin{itemize}
	\item 
	Using the identity
	\[
	\ip_d(\bx,\by) = \ip_d(\bx+\bx',\by+\by') + \ip_d(\bx',\by') + \ip_d(\bx+\bx',\by') + \ip_d(\bx',\by+\by')  \mod 2,
	\]
	which holds for all $\bx,\by,\bx',\by'\in\{0,1\}^{d}$, we can generate uniformly random binary vectors $\bx',\by'\in\{0,1\}^d$, and replace $\bx$ and $\by$ with the vectors $(\bx+\bx',\bx',\bx+\bx',\bx')$ and $(\by+\by',\by',\by',\by+\by')$, respectively, while keeping the IP mod 2 output unchanged. This additional randomness is useful for handling cases where the inputs have a lot of structure (e.g., when $\bx,\by$ are the all-zero or all-one vectors).
	\item
	We can pad our modified inputs with $\Ocal(d)$ many uniformly generated bits, such that pairs where $x_i=y_i=1$ are conditioned to be an even number. Due to this conditioning, we have that the inner product mod 2 value is unchanged. This padding greatly increases the randomness of the input.
	\item
	Lastly, since addition mod 2 is commutative, we can sample a random permutation, and apply it to our modified input, once again altering our input while keeping its IP mod 2 value unchanged. This allows us to add valuable randomness to our input, since permuted vectors with an almost equal number of pairs of the form $x_i=y_i=0$; $x_i=0,y_i=1$; $x_i=1,y_i=0$; $x_i=y_i=1$ are much closer to uniformly sampled vectors.
\end{itemize}
Combining the above into a single randomization process, we have a new, higher dimensional input (but with at most a linear blow-up), which has the same IP mod 2 value, yet whose probability distribution is close (in a certain sense) to a uniformly random input from the higher dimensional space. The bulk of the technical analysis in our proof consists of proving that this process results in a distribution close enough to uniform, so as to incur at most a constant additional loss in our approximation. This requires a very careful and tight analysis of the distribution. 

After constructing the architecture which performs the above randomization process, we repeat the process polynomially many times, we concatenate the obtained hidden layers which perform this computation, and we use the output neuron to average their outputs. This results in a concentration of measure, which makes the approximation error much more uniformly spread over the Boolean hypercube. By a union bound and the probabilistic method, we can find realizations of our random construction that achieve this, and modify our network to incorporate these realizations without changing the architecture of the network $\Ncal''$. This can be done since the operations of inverting a bit of the input or permuting the inputs are linear operations that can be absorbed in the weights of the hidden layer.


\subsubsection{Step 3: Constructing a threshold circuit computing $\ip_d$}

Following the previous two steps, we now have a depth 2, $\sigma$-network $\Ncal''$, such that $\round(\Ncal''(\bx,\by)) = \ip_d(\bx,\by)$ for all $\bx,\by\in\{0,1\}^{d}$. The final step in our proof is to use this architecture to construct a depth 2 threshold circuit with the same property. To this end, we first need to approximate $\sigma$ to arbitrary accuracy using a depth 2 threshold network on a compact domain. This can be done by constructing a piecewise linear approximation as follows: Beginning with the leftmost point in the domain of approximation, we choose a constant function which coalesces with $\sigma$ at this point, and we extend it until its distance from $\sigma$ deviates from our target accuracy, in which case we make a `jump' to a different constant value, continuing in this manner until the approximation is complete. Since \asmref{asm:lb} guarantees that $\sigma$ is `well-behaved' in the sense that it cannot change its output too wildly, we have that we cannot perform too many such jumps;
namely, $\sigma$ can be approximated by a piecewise linear function with not too many discontinuities. Replacing each $\sigma$ with a moderately sized depth 2 threshold network, we obtain a moderately sized depth 2 threshold network $\Ncal'''$ that satisfies $\round(\Ncal'''(\bx,\by)) = \ip_d(\bx,\by)$. Lastly, we turn this threshold network into a threshold circuit by adding a threshold activation on the output neuron. It is a classic result in circuit complexity that $\ip_d$ cannot be computed by a depth 2 threshold circuit with bounded weights, unless its width is exponential in $d$ \citep{hajnal1993threshold}. This implies a lower bound on the width of $\Ncal$, from which the theorem follows.

\section{Upper Bound}\label{sec:ub}

Having presented our lower bound, in this section, we turn to complement it with a relatively straightforward upper bound for depth 3 networks. We also provide a concrete example in the case where $\sigma$ is the ReLU activation function, in which our required width and magnitude of the weights provides a stronger result than the bounds guaranteed in our theorem.

We now formally state our upper bound result below.

\begin{theorem}\label{thm:ub}
	Let $\varepsilon>0$, and suppose that $\sigma$ satisfies \asmref{asm:ub}. Then, there exists a depth 3, width $\Ocal\p{\frac{d^2}{\varepsilon}}$ $\sigma$-network $\Ncal$, with weights bounded in magnitude by $\Ocal\p{\frac{d}{\varepsilon^2}}$, such that
	\[
	\sup_{\bx\in\Acal_{4d}}\abs{\Ncal(\bx) - f_d(\bx)} \le \varepsilon.
	\]
\end{theorem}

The proof of the theorem, which appears in \appref{app:ub}, utilizes the fact that $f_d$ can be approximated efficiently by composing two different simple functions, which focus on the last $2d$ coordinates of the input. Since a depth 3 network has two hidden layers with non-linear $\sigma$ activations, we are able to use each hidden layer to compute each function, and obtain the desired approximation.

We remark that we provide our positive approximation result in terms of the $L_{\infty}$ norm rather than the $L_2$ norm with respect to the uniform distribution supported on $\Acal_{4d}$ as our lower bound does. Since $L_{\infty}$ approximation is more stringent than $L_2$, this provides a more general result which implies an $L_2$ approximation of $f_d(\cdot,\cdot)$ with respect to \emph{any} distribution supported on $\Acal_{4d}$.

To give a more concrete example of an approximation obtained by our theorem, below we specify how this construction can be done when $\sigma$ is the ReLU activation.

\begin{example}\label{ex:relu}
	Let $\relu{x}\coloneqq\max\{0,x\}$ denote the ReLU activation function, and define 
	\[
	t_d(x)\coloneqq\relu{z}+\sum_{i=1}^{d}2\cdot(-1)^i\relu{z-i}.
	\]
	Then, for an input $\bx'\coloneqq(\bz,\bx,\by)\in\reals^{2d}\times\reals^d\times\reals^d$, the network given by
	\[
	\Ncal(\bz,\bx,\by)\coloneqq t_d\p{\sum_{i=1}^d 3\sqrt{d}\relu{4x_i+4y_i-5} - 3\sqrt{d}\relu{4x_i+4y_i-6}}
	\]
	satisfies
	\[
	\sup_{\bx'\in\Acal_{d}}\abs{\Ncal(\bx') - f_d(\bx')} = 0.
	\]
\end{example}

It is straightforward to verify that $\Ncal$ is a depth 3, width $2d$ ReLU network, with weights bounded in magnitude by $\Ocal(d)$. Moreover, a simple computation shows that $\Ncal$ coalesces with $f_d$ on $\Acal_{4d}$. Lastly, this approximation is also sparse, in the sense that it requires only $\Ocal(d)$ neuron connections (see \figref{fig:relu} for a visualization of the functions computed in each layer).

    \begin{figure}[]
        \centering
        \begin{subfigure}[b]{0.32\textwidth}
            \centering
            \begin{tikzpicture}[scale=0.67]
                \begin{axis}[zmin=-0.1,zmax=1.1,view={30}{40},grid=both]
                    \addplot3 [surf,unbounded coords=jump,shader=interp,domain=0:1,y domain=0:1,samples=90] {max(0,4*x+4*y-5) - max(0,4*x+4*y-6)};
                \end{axis}
            \end{tikzpicture}
            \caption{}
            \label{fig:a}
        \end{subfigure}
        \hfill
        \begin{subfigure}[b]{0.67\textwidth}
            \centering
            \begin{tikzpicture}
            \begin{axis}[y=2.25cm,x=2.25cm,xmin=-0.15,xmax=4.15]
                \addplot[height=1cm,color=blue,domain=-1:7,very thick] plot {max(0,x) - 2*max(0,x-1) + 2*max(0,x-2) - 2*max(0,x-3) + 2*max(0,x-4) - 2*max(0,x-5) + 2*max(0,x-6)};
            \end{axis}
            \end{tikzpicture}
            \caption{}
            \label{fig:b}
        \end{subfigure}
        \caption{Computing $f_d$ using a depth 3 ReLU network. Subfigure~\ref{fig:a} plots the function $(x,y)\mapsto\relu{4x+4y-5} - \relu{4x+4y-6}$, which equals $\AND(\round(x),\round(y))$ for all $x,y\in[0,0.25]\cup[0.75,1]$. Subfigure~\ref{fig:b} plots the function $x\mapsto t_d(x)$, defined in Example~\ref{ex:relu}. When composing the latter with a scaled sum of the former, iterating over all pairs of coordinates $x_i,y_i$, we obtain a function that coalesces with $f_d$ on $\Acal_{4d}$. Best viewed in color.}
        \label{fig:relu}
    \end{figure}
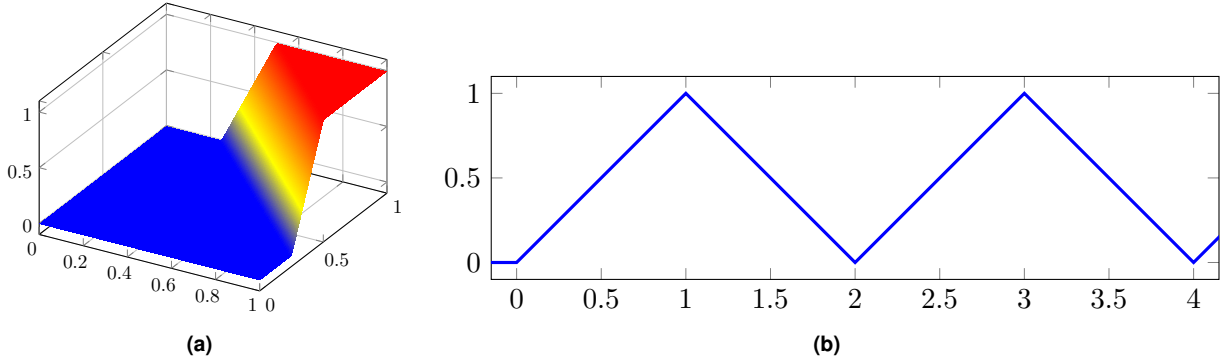

Following the works \citep{malach2019deeper,malach2021connection}, which show that some functions that were used to prove approximation lower bounds for neural networks cannot be learned efficiently using standard methods, \citet{safran2022optimization} have shown an optimization-based separation result where the deeper architecture can provably learn the efficient representation from finite data, using standard techniques such as gradient descent. It is interesting to note that Example~\ref{ex:relu} provides a simple, linear in size and sparse approximation of $f_d$. This simplicity is much desired, since it may suggest that similarly to \citet{safran2022optimization}, learning this representation from finite data using a standard learning algorithm is tractable, and despite the simplicity of this ReLU approximation, our lower bound for this function provides a separation which is stronger than previously known results. We leave the study of proving such a stronger optimization-based separation result as an intriguing future work direction.

\ifanonymize 
\else
    \subsection*{Acknowledgements.}

    Paul Valiant is partially supported by NSF award CCF-2127806. We thank Srikanth Srinivasan for useful discussions.
\fi

\bibliographystyle{plainnat}
\bibliography{citations}

\begin{thebibliography}{30}
\providecommand{\natexlab}[1]{#1}
\providecommand{\url}[1]{\texttt{#1}}
\expandafter\ifx\csname urlstyle\endcsname\relax
  \providecommand{\doi}[1]{doi: #1}\else
  \providecommand{\doi}{doi: \begingroup \urlstyle{rm}\Url}\fi

\bibitem[Boucheron et~al.(2005)Boucheron, Bousquet, and
  Lugosi]{boucheron2005theory}
St{\'e}phane Boucheron, Olivier Bousquet, and G{\'a}bor Lugosi.
\newblock Theory of classification: A survey of some recent advances.
\newblock \emph{ESAIM: probability and statistics}, 9:\penalty0 323--375, 2005.

\bibitem[Cybenko(1989)]{cybenko1989approximation}
George Cybenko.
\newblock Approximation by superpositions of a sigmoidal function.
\newblock \emph{Mathematics of control, signals and systems}, 2\penalty0
  (4):\penalty0 303--314, 1989.

\bibitem[Daniely(2017)]{daniely2017depth}
Amit Daniely.
\newblock Depth separation for neural networks.
\newblock In \emph{Conference on Learning Theory}, pages 690--696, 2017.

\bibitem[Eldan and Shamir(2016)]{eldan2016power}
Ronen Eldan and Ohad Shamir.
\newblock The power of depth for feedforward neural networks.
\newblock In \emph{Conference on learning theory}, pages 907--940. PMLR, 2016.

\bibitem[Golowich et~al.(2018)Golowich, Rakhlin, and Shamir]{golowich2018size}
Noah Golowich, Alexander Rakhlin, and Ohad Shamir.
\newblock Size-independent sample complexity of neural networks.
\newblock In \emph{Conference On Learning Theory}, pages 297--299. PMLR, 2018.

\bibitem[Hajnal et~al.(1993)Hajnal, Maass, Pudl{\'a}k, Szegedy, and
  Tur{\'a}n]{hajnal1993threshold}
Andr{\'a}s Hajnal, Wolfgang Maass, Pavel Pudl{\'a}k, Mario Szegedy, and
  Gy{\"o}rgy Tur{\'a}n.
\newblock Threshold circuits of bounded depth.
\newblock \emph{Journal of Computer and System Sciences}, 46\penalty0
  (2):\penalty0 129--154, 1993.

\bibitem[He et~al.(2016)He, Zhang, Ren, and Sun]{he2016deep}
Kaiming He, Xiangyu Zhang, Shaoqing Ren, and Jian Sun.
\newblock Deep residual learning for image recognition.
\newblock In \emph{Proceedings of the IEEE conference on computer vision and
  pattern recognition}, pages 770--778, 2016.

\bibitem[Hornik et~al.(1989)Hornik, Stinchcombe, and
  White]{hornik1989multilayer}
Kurt Hornik, Maxwell Stinchcombe, and Halbert White.
\newblock Multilayer feedforward networks are universal approximators.
\newblock \emph{Neural networks}, 2\penalty0 (5):\penalty0 359--366, 1989.

\bibitem[Hsu et~al.(2021)Hsu, Sanford, Servedio, and
  Vlatakis-Gkaragkounis]{hsu2021approximation}
Daniel Hsu, Clayton~H Sanford, Rocco Servedio, and Emmanouil~Vasileios
  Vlatakis-Gkaragkounis.
\newblock On the approximation power of two-layer networks of random relus.
\newblock In \emph{Conference on Learning Theory}, pages 2423--2461. PMLR,
  2021.

\bibitem[Jenssen et~al.(2018)Jenssen, Joos, and Perkins]{jenssen2018kissing}
Matthew Jenssen, Felix Joos, and Will Perkins.
\newblock On kissing numbers and spherical codes in high dimensions.
\newblock \emph{Advances in Mathematics}, 335:\penalty0 307--321, 2018.

\bibitem[LeCun et~al.(2015)LeCun, Bengio, and Hinton]{lecun2015deep}
Yann LeCun, Yoshua Bengio, and Geoffrey Hinton.
\newblock Deep learning.
\newblock \emph{nature}, 521\penalty0 (7553):\penalty0 436--444, 2015.

\bibitem[Leshno et~al.(1993)Leshno, Lin, Pinkus, and
  Schocken]{leshno1993multilayer}
Moshe Leshno, Vladimir~Ya Lin, Allan Pinkus, and Shimon Schocken.
\newblock Multilayer feedforward networks with a nonpolynomial activation
  function can approximate any function.
\newblock \emph{Neural networks}, 6\penalty0 (6):\penalty0 861--867, 1993.

\bibitem[Liang and Srikant(2016)]{liang2016deep}
Shiyu Liang and Rayadurgam Srikant.
\newblock Why deep neural networks for function approximation?
\newblock \emph{arXiv preprint arXiv:1610.04161}, 2016.

\bibitem[Malach and Shalev-Shwartz(2019)]{malach2019deeper}
Eran Malach and Shai Shalev-Shwartz.
\newblock Is deeper better only when shallow is good?
\newblock \emph{Advances in Neural Information Processing Systems}, 32, 2019.

\bibitem[Malach et~al.(2021)Malach, Yehudai, Shalev-Schwartz, and
  Shamir]{malach2021connection}
Eran Malach, Gilad Yehudai, Shai Shalev-Schwartz, and Ohad Shamir.
\newblock The connection between approximation, depth separation and
  learnability in neural networks.
\newblock In \emph{Conference on Learning Theory}, pages 3265--3295. PMLR,
  2021.

\bibitem[Martens et~al.(2013)Martens, Chattopadhya, Pitassi, and
  Zemel]{martens2013representational}
James Martens, Arkadev Chattopadhya, Toni Pitassi, and Richard Zemel.
\newblock On the representational efficiency of restricted boltzmann machines.
\newblock \emph{Advances in Neural Information Processing Systems}, 26, 2013.

\bibitem[Mukherjee and Basu(2017)]{mukherjee2017lower}
Anirbit Mukherjee and Amitabh Basu.
\newblock Lower bounds over boolean inputs for deep neural networks with relu
  gates.
\newblock \emph{arXiv preprint arXiv:1711.03073}, 2017.

\bibitem[Nichani et~al.(2023)Nichani, Damian, and Lee]{nichani2023provable}
Eshaan Nichani, Alex Damian, and Jason~D Lee.
\newblock Provable guarantees for nonlinear feature learning in three-layer
  neural networks.
\newblock \emph{arXiv preprint arXiv:2305.06986}, 2023.

\bibitem[Poggio et~al.(2017)Poggio, Mhaskar, Rosasco, Miranda, and
  Liao]{poggio2017and}
Tomaso Poggio, Hrushikesh Mhaskar, Lorenzo Rosasco, Brando Miranda, and Qianli
  Liao.
\newblock Why and when can deep-but not shallow-networks avoid the curse of
  dimensionality: a review.
\newblock \emph{International Journal of Automation and Computing}, 14\penalty0
  (5):\penalty0 503--519, 2017.

\bibitem[Ren et~al.(2023)Ren, Zhou, and Ge]{ren2023depth}
Yunwei Ren, Mo~Zhou, and Rong Ge.
\newblock Depth separation with multilayer mean-field networks.
\newblock \emph{arXiv preprint arXiv:2304.01063}, 2023.

\bibitem[Safran and Lee(2022)]{safran2022optimization}
Itay Safran and Jason Lee.
\newblock Optimization-based separations for neural networks.
\newblock In \emph{Conference on Learning Theory}, pages 3--64. PMLR, 2022.

\bibitem[Safran and Shamir(2017)]{safran2017depth}
Itay Safran and Ohad Shamir.
\newblock Depth-width tradeoffs in approximating natural functions with neural
  networks.
\newblock In \emph{International conference on machine learning}, pages
  2979--2987. PMLR, 2017.

\bibitem[Safran et~al.(2019)Safran, Eldan, and Shamir]{safran2019depth}
Itay Safran, Ronen Eldan, and Ohad Shamir.
\newblock Depth separations in neural networks: what is actually being
  separated?
\newblock In \emph{Conference on Learning Theory}, pages 2664--2666. PMLR,
  2019.

\bibitem[Safran et~al.(2024)Safran, Reichman, and Valiant]{safran2024many}
Itay Safran, Daniel Reichman, and Paul Valiant.
\newblock How many neurons does it take to approximate the maximum?
\newblock In \emph{Proceedings of the 2024 Annual ACM-SIAM Symposium on
  Discrete Algorithms (SODA)}, pages 3156--3183. SIAM, 2024.

\bibitem[Shalev-Shwartz and Ben-David(2014)]{shalev2014understanding}
Shai Shalev-Shwartz and Shai Ben-David.
\newblock \emph{Understanding machine learning: From theory to algorithms}.
\newblock Cambridge university press, 2014.

\bibitem[Telgarsky(2016)]{telgarsky2016benefits}
Matus Telgarsky.
\newblock Benefits of depth in neural networks.
\newblock In \emph{Conference on learning theory}, pages 1517--1539. PMLR,
  2016.

\bibitem[Vardi and Shamir(2020)]{vardi2020neural}
Gal Vardi and Ohad Shamir.
\newblock Neural networks with small weights and depth-separation barriers.
\newblock \emph{arXiv preprint arXiv:2006.00625}, 2020.

\bibitem[Vardi et~al.(2021)Vardi, Reichman, Pitassi, and Shamir]{vardi2021size}
Gal Vardi, Daniel Reichman, Toniann Pitassi, and Ohad Shamir.
\newblock Size and depth separation in approximating benign functions with
  neural networks.
\newblock In \emph{Conference on Learning Theory}, pages 4195--4223. PMLR,
  2021.

\bibitem[Venturi et~al.(2022)Venturi, Jelassi, Ozuch, and
  Bruna]{venturi2022depth}
Luca Venturi, Samy Jelassi, Tristan Ozuch, and Joan Bruna.
\newblock Depth separation beyond radial functions.
\newblock \emph{The Journal of Machine Learning Research}, 23\penalty0
  (1):\penalty0 5309--5364, 2022.

\bibitem[Yarotsky(2017)]{yarotsky2017error}
Dmitry Yarotsky.
\newblock Error bounds for approximations with deep relu networks.
\newblock \emph{Neural Networks}, 94:\penalty0 103--114, 2017.

\end{thebibliography}

\appendix

\section{Proofs}

\subsection{Proof of \thmref{thm:lb}}\label{app:lb}

Before we begin, we first state a relaxation of \asmref{asm:lb} which we will use in our proof. Albeit slightly technical, it greatly extends the types of activation functions for which our lower bound holds. The following definitions are required for understanding the assumption statement: Given a function $f:\reals\to\reals$, we define its \emph{total variation} on the interval $[a,b]\subseteq\reals$ as $V_a^b(f)\coloneqq\sup_{P}\sum_{i=1}^{n_P}\abs{f(x_i)-f(x_{i-1})}$, where the supremum is taken over all the possible partitions of the interval $[a,b]$. For a function class $\Fcal$ and sample $\bx_1,\ldots,\bx_n$, the (empirical) Rademacher complexity of $\Fcal$ is given by $\Rcal_n(\Fcal)\coloneqq\Exp\pcc{\sup_{f\in\Fcal}|\frac1n\sum_{i=1}^n\xi_if(\bx_i)|}$, where the expectation is taken over $\xi_1,\ldots,\xi_n$ which are i.i.d.\ Rademacher random variables.

\begin{assumption}[Lower bound -- generalized]\label{asm:lb_relaxed}
	The activation function $\sigma$ is (Lebesgue) measurable and satisfies
	\[
	   \abs{\sigma(x)}\le C_{\sigma}\p{1+\abs{x}^{\alpha_{\sigma}}},
	\]
        and
	\[
	   V_a^b(\sigma) \le C_{\sigma}\p{1+(|a|+|b|)^{\alpha_{\sigma}}},
	\]
    for all $x\in\reals$ and real $a<b$; and the function class $\Ncal_{w,C,d}$ of $d$-dimensional $\sigma$-networks of depth 2, width at most $w$, and weights bounded by $C$, satisfies
    \[
        \Rcal_n(\Ncal_{w,C,d}) \le \frac{\p{w\cdot C\cdot d}^{\alpha_{\sigma}}}{\sqrt{n}},
    \]
    for any dataset of size $n$ contained inside $[0,1]^d$, and for some constants $C_{\sigma},\alpha_{\sigma}$, which depend solely on $\sigma$.
\end{assumption}

Our above assumption is very mild, and as previously discussed, it is satisfied by many commonly used activations such as ReLU, threshold and sigmoidal activations as a particular case. The fact that these activation functions satisfy the first two requirements is straightforward, and an appropriate Rademacher bound for thresholds follows from \citet[Thm.~20.6]{shalev2014understanding} and \citet[Eq.~(6)]{boucheron2005theory}, and for $\Ocal(1)$-Lipschitz activation functions from \citet{golowich2018size}. The boundedness of $|\sigma(x)|$ is a standard assumption when proving approximation lower bounds, and it is also used in \citet{eldan2016power} for example. Since our proof is based on a reduction to threshold circuits, our technique fails if we consider certain activation functions that are highly oscillatory, since there are such pathological activations that can be used to compute any Boolean function $f:\{0,1\}^d\to\{0,1\}$, even with just a single neuron. In light of this, we also make an assumption that the total variation of the activation function is polynomially bounded on a compact domain. Lastly, our assumption on the boundedness of the Rademacher complexity merely requires that the depth 2 networks which employ $\sigma$ activations are statistically learnable. Such an assumption is again very mild, and holds for essentially any activation function which is used in practice.

Before we can prove the theorem, we will first state and prove several auxiliary lemmas that will be used later on. In what follows, we use the notation $\norm{\Dcal}_2$ to denote the $L_2$ norm of a discrete distribution $\Dcal$. Namely, for a random variable $X\sim\Dcal$ sampled from a sample space $\Xcal$, we have
\[
\norm{\Dcal}_2 \coloneqq \sqrt{\sum_{x\in\Xcal}\p{\pr_{x\sim\Dcal}\pcc{X=x}}^2}.
\]

The following lemma provides an upper bound on certain subsets of the $L_2$ norm of the distribution we analyze in our reduction scheme.

\begin{lemma}\label{lem:binom-square-ratio}
	For $d,D$ positive integers divisible by 4, and given $d_1,d_2,d_3,d_4\geq 0$ with $d_1+d_2+d_3+d_4=d$ then we have:
	
	\begin{align}\label{eq:binom-square-ratio}
		\sum_{(D_1, D_2,D_3,D_4): \sum_i D_i = D} \frac{{D \choose D_1,D_2,D_3,D_4}^2}{{D+d \choose D_1+d_1,D_2+d_2,D_3+d_3,D_4+d_4}} \leq e^{\frac{4}{D}\sum_{i=1}^4 (d_i-\frac{d}{4})^2}\cdot \p{1+\frac{d}{D}}^{3/2}\cdot 4^{D-d}
	\end{align}
	
\end{lemma}
\begin{proof}
	We start by comparing $\frac{(D_i+d_i)!}{D_i!}$ to $\frac{(D/4+d_i)!}{D/4!}$. For $D_i \geq D/4$ the ratio of these two expressions equals $\prod_{j=D/4+1}^{D_i} \frac{j+d_i}{j}\leq \prod_{j=D/4+1}^{D_i} \frac{D/4+d_i}{D/4} =(1+\frac{4d_i}{D})^{D_i-D/4}$. On the other hand, for $D_i<D/4$ that ratio of our two expressions equals $\prod_{j=D_i+1}^{D/4} \frac{j}{j+d_i}\leq \prod_{j=D_i+1}^{D/4}\frac{D/4}{D/4+d_i} =(1+\frac{4d_i}{D})^{D_i-D/4}$. Thus in all cases, $\frac{(D_i+d_i)!}{D_i!}\leq (1+\frac{4d_i}{D})^{D_i-D/4}\frac{(D/4+d_i)!}{D/4!}$. We apply this inequality to bound the terms on the left hand side of \eqref{eq:binom-square-ratio}:
	\begin{align}\label{eq:ratio-bound}
		\frac{{D \choose D_1,D_2,D_3,D_4}}{{D+d \choose D_1+d_1,D_2+d_2,D_3+d_3,D_4+d_4}} &= \frac{D!}{(D+d)!}\frac{(D_1+d_1)!}{D_1!} \frac{(D_2+d_2)!}{D_2!} \frac{(D_3+d_3)!}{D_3!} \frac{(D_4 +d_4)!}{D_4! }\\
		& \leq \frac{{D \choose D/4,D/4,D/4,D/4}}{{D+d \choose D/4+d_1,D/4+d_2,D/4+d_3,D/4+d_4}} \prod_{i=1}^4 (1+\frac{4d_i}{D})^{D_i-D/4}\notag
	\end{align}
	
	We will evaluate the sum in Equation \ref{eq:binom-square-ratio} and using the identity \[(p_1 +p_2+p_3+p_4)^D = \sum_{(D_1, D_2,D_3,D_4): \sum_i D_i = D}{D \choose D_1,D_2,D_3,D_4 } p_1^{D_1} p_2^{D_2} p_3^{D_3} p_4^{D_4}\] along with the above equation we can write,
	
	\begin{align*}
		\sum_{(D_1, D_2,D_3,D_4): \sum_i D_i = D} \frac{{D \choose D_1,D_2,D_3,D_4}^2}{{D+d \choose D_1+d_1,D_2+d_2,D_3+d_3,D_4+d_4}} \\
		= \sum_{(D_1, D_2,D_3,D_4): \sum_i D_i = D} {D \choose D_1,D_2,D_3,D_4} \frac{{D \choose D_1,D_2,D_3,D_4}}{{D+d \choose D_1+d_1,D_2+d_2,D_3+d_3,D_4+d_4}} \\
		{\leq} \frac{{D \choose D/4,D/4,D/4,D/4}}{{D+d \choose D/4+d_1,D/4+d_2,D/4+d_3,D/4+d_4}} \left(\prod_{i=1}^4 (1+\frac{4d_i}{D})^{-D/4}\right)
		\sum_{(D_1, D_2,D_3,D_4): \sum_i D_i = D} {D \choose D_1,D_2,D_3,D_4} \prod_{i=1}^4 (1+\frac{4d_i}{D})^{D_i} \\
		= \frac{{D \choose D/4,D/4,D/4,D/4}}{{D+d \choose D/4+d_1,D/4+d_2,D/4+d_3,D/4+d_4}} \left(\prod_{i=1}^4 (1+\frac{4d_i}{D})^{-D/4}\right)
		(4+\frac{4\sum_i d_i}{D})^{D} 
	\end{align*}
	
	Since the second derivative of the function $\log( x)$ is $\geq -1$ for inputs $x\geq 1$, we use the Lagrange remainder form of the Taylor expansion around any $\ell\geq 1$ to lower bound the logarithm function for all $x\geq 1$ as $\log(x)\geq \log(\ell)+\frac{1}{\ell}(x-\ell) -\frac{1}{2}(x-\ell)^2$. We use this to bound the last two terms above, bounding $\log(x)$ for $x=1+\frac{4d_i}{D}$, centering our approximation at $\ell=1+\frac{d}{D}$, and using the fact that $\sum_i d_i=d$:
	\begin{align*}\left(\prod_{i=1}^4 (1+\frac{4d_i}{D})^{-D/4}\right)
		(4+\frac{4\sum_i d_i}{D})^{D} &\leq \left(\prod_{i=1}^4 \left(e^{\log(1+\frac{d}{D})+\frac{4 d_i-d}{D(1+\frac{d}{D})}-\frac{1}{2}(\frac{4 d_i-d}{D})^2}\right)^{-D/4}\right)(4+\frac{4\sum_i d_i}{D})^{D}\\
		&= 4^D e^{\sum_{i=1}^4 2\frac{(d_i-\frac{d}{4})^2}{D}}
	\end{align*}
	
	We then turn to the first term, which we reexpress as
	\[\frac{{D \choose D/4,D/4,D/4,D/4}}{{D+d \choose D/4+d_1,D/4+d_2,D/4+d_3,D/4+d_4}} = \frac{D!}{(D+d)!}\frac{(D/4+d_1)!}{(D/4)!} \frac{(D/4+d_2)!}{(D/4)!} \frac{(D/4+d_3)!}{(D/4)!} \frac{(D/4 +d_4)!}{(D/4)! }\]
	
	We bound $(D/4+d_i)!$ by reexpressing it via the Gamma function as $\Gamma(1+D/4+d_i)$. We then use the Lagrange remainder form of the Taylor expansion of the function $f(x):=\log\Gamma(x)$ around $\ell=1+\frac{D}{4}+\frac{d}{4}$ to conclude that there exists $y$ between $\ell$ and $x$ 
	such that \[f(x)=f(\ell)+(x-\ell)f'(\ell)+\frac{1}{2}f''(y)(x-\ell)^2\]
	
	The second derivative of the $\log\Gamma$ function at $x$ is bounded by $\frac{x+1}{x^2}\leq \frac{x+1}{x^2-1}=\frac{1}{x-1}$. Thus, letting $x_i=1+D/4+d_i$, we have, that there exist $y_i$ between $x_i$ and $\ell$ for which
	\begin{align*}\prod_{i=1}^4 (D/4+d_i)! &\leq \prod_{i=1}^4 e^{f(\ell)+(x_i-\ell)f'(\ell)+\frac{1}{2}f''(y_i)(x_i-\ell)^2}\\
		& = (D/4+d/4)!^4 e^{\frac{1}{2}\sum_{i=1}^4 f''(y_i)(d_i-\frac{d}{4})^2}\\
		&\leq (D/4+d/4)!^4 e^{\frac{2}{D}\sum_{i=1}^4 (d_i-\frac{d}{4})^2}
	\end{align*}
	where the equality makes use of the fact that the middle (linear in $x_i$) term vanishes since $\sum_{i=1}^4 (x_i-\ell)=0$; the last inequality makes use of the fact that both $x_i,\ell\geq 1+D/4$, and thus, since $y_i$ is between these, we have $y_i\geq 1+D/4$, and hence our bound on the second derivative of $\log\Gamma(y)$ is at most $\frac{4}{D}$.
	
	Thus, overall, we have shown
	\begin{align*}\sum_{(D_1, D_2,D_3,D_4): \sum_i D_i = D} \frac{{D \choose D_1,D_2,D_3,D_4}^2}{{D+d \choose D_1+d_1,D_2+d_2,D_3+d_3,D_4+d_4}}& \leq 4^D e^{\sum_{i=1}^4 2\frac{(d_i-\frac{d}{4})^2}{D}} \frac{D!(D/4+d/4)!^4}{(D+d)! (D/4)!^4} e^{\frac{2}{D}\sum_{i=1}^4 (d_i-\frac{d}{4})^2}\\
		&= 4^D  \frac{D!(D/4+d/4)!^4}{(D+d)! (D/4)!^4} e^{\frac{4}{D}\sum_{i=1}^4 (d_i-\frac{d}{4})^2}
	\end{align*}
	
	Finally, we bound the ratio of factorials as follows. Define the function $f$ on integer input $f(j)\coloneqq 4^{-4j}{4j\choose j,j,j,j}=4^{-4j}\frac{(4j)!}{j!^4}$. We will show that for any integers $k\geq j$ we have $\frac{f(j)}{f(k)}\leq (j/k)^{3/2}$. We have $\frac{f(j)}{f(j+1)}=\frac{(4j+4)^3}{(4j+1)(4j+2)(4j+3)}$. Expressing each term in the denominator as a weighted average of $4j$ and $4j+4$ and applying the (weighted) AM-GM inequality to each term on the denominator separately, we thus lower bound the denominator by $(4j)^{3/2}(4j+4)^{3/2}$ and thus upper bound $\frac{f(j)}{f(j+1)}\leq \frac{(4j+4)^{3/2}}{(4j)^{3/2}}=(\frac{j+1}{j})^{3/2}$. Multiplying this bound for all numbers between $j$ and $k$ yields $\frac{f(j)}{f(k)}\leq (j/k)^{3/2}$ as claimed.
	
	Thus $\frac{D!(D/4+d/4)!^4}{(D+d)! (D/4)!^4}\leq \p{\frac{D+d}{D}}^{3/2}$. And we have derived \eqref{eq:binom-square-ratio} as claimed.
\end{proof}

The following lemma bounds the moment generating function of a certain distribution, which arises in our analysis of the reduction scheme.

\begin{lemma}\label{lem:xor-identity}
	For any dimension $d$ and any vectors $\bx,\by\in\{0,1\}^d$, consider the process of picking random vectors $\bx',\by'\leftarrow\{0,1\}^d$ and outputting 
	\begin{align*}
		A\coloneqq(\bx+\bx',\bx',\bx+\bx',\bx'),\\
		B\coloneqq(\by+\by',\by',\by',\by+\by'),
	\end{align*}
	where all additions are mod $2$. Among these $4d$ entries, let $d_1$ count the number of indices $j\in\{1,\ldots,4d\}$ where $A_j=B_j=0$; let $d_2$ count the entries where $A_j=0,B_j=1$; let $d_3$ count the entries where $A_j=1,B_j=0$; and let $d_4$ count the entries where $A_j=B_j=1$. Then 
	\[
	\Exp_{d_1,d_2,d_3,d_4}\pcc{e^{s\sum_{i=1}^4 (d_i-d)^2}}\leq \left(\frac{1}{1-24ds}\right)^2
	\]
\end{lemma}
\begin{proof}
	Our overall analysis technique will be to first compute the moment generating function (MGF) of the distribution of $(d_1,d_2,d_3,d_4)$, centered at its mean, $(d,d,d,d)$. Explicitly, since we have a 4-dimensional distribution, the MGF has 4 parameters:
	\[
	M(t_1,t_2,t_3,t_4)\coloneqq \Exp_{d_1,d_2,d_3,d_4}\pcc{e^{\sum_{i=1}^4 t_i (d_i-d)}}
	\]
	
	Crucially, the MGF of the sum of independent random vectors multiplies; thus we will compute the MGF for the contribution to $(d_1,d_2,d_3,d_4)$ from each location $A_j,B_j$ separately.
	
	After we have computed the MGF, we will use this to compute the quantity in the lemma statement, $\Exp_{d_1,d_2,d_3,d_4}\pcc{e^{s\sum_{i=1}^4 (d_i-d)^2}}$. Since the MGF of the Gaussian $\mathcal{N}(0,2s)$ equals $e^{st^2}$, the MGF for the corresponding 4-dimensional Gaussian, after a slight change of variables, yields the relation that, for any $d_1,d_2,d_3,d_4$: \[\Exp_{(t_1,t_2,t_3,t_4)\leftarrow \mathcal{N}(0,2s)}[e^{\sum_{i=1}^4 t_i (d_i-d)}]=e^{s\sum_{i=1}^4 (d_i-d)^2}\]
	and thus that 
	\begin{equation}
		\label{eq:moment-gaussian}
		\Exp_{(t_1,t_2,t_3,t_4)\leftarrow \mathcal{N}(0,2s)}[M(t_1,t_2,t_3,t_4)]=\Exp_{d_1,d_2,d_3,d_4}[e^{s\sum_{i=1}^4 (d_i-d)^2}]
	\end{equation}
	
	Consider, for some index $j$, the 4 cases for the pair $x_j,y_j$, and the 4-tuple of pairs of entries $(x_j+x'_j,y_j+y'_j),(x'_j,y'_j),(x_j+x'_j,y'_j),(x'_j,y_j+y'_j)$ they induce in $A,B$.
	
	If $x_j=0,y_j=0$, then the 4-tuple of pairs $(x_j+x'_j,y_j+y'_j),(x'_j,y'_j),(x_j+x'_j,y'_j),(x'_j,y_j+y'_j)$ produced in the vectors $A,B$ contains simply 4 copies of the randomly chosen pair of bits $(x'_j,y'_j)$. Thus the contribution to the counts $d_1,d_2,d_3,d_4$ will be uniformly randomly chosen among the 4-tuples $(4,0,0,0), (0,4,0,0), (0,0,4,0), (0,0,0,4)$. The moment generating function of this uniform distribution over 4 possibilities is just the sum of the 4 terms in the definition of the MGF; as usual, we center the MGF of this portion of the distribution at its mean, which here is $(1,1,1,1)$:
	\begin{align*}M_j(t_1,t_2,t_3,t_4)=&\frac{1}{4}e^{(4-1)t_1+(0-1)t_2+(0-1)t_3+(0-1)t_4}+\frac{1}{4}e^{(0-1)t_1+(4-1)t_2+(0-1)t_3+(0-1)t_4}\\
		&+\frac{1}{4}e^{(0-1)t_1+(0-1)t_2+(4-1)t_3+(0-1)t_4}+\frac{1}{4}e^{(0-1)t_1+(0-1)t_2+(0-1)t_3+(4-1)t_4}\end{align*}
	
	We can loosely bound this by $e^{6\sum_{i=1}^4 t_i^2}$ as follows. Each of the 4 exponential terms is the exponential of the inner product of $(t_1,t_2,t_3,t_4)$ with a vector $v$ of length $\sqrt{12}$. Thus, for any unit 4-vector $u$, we have that the exponential $e^{\langle(t_1,t_2,t_3,t_4),v\rangle}$ has second derivative in direction $u$ that is at most $12$ times the value of the exponential $e^{\langle(t_1,t_2,t_3,t_4),v\rangle}$. Since this property---of having directional second derivative in direction $u$ that is at most 12 times the function value---is preserved under scaling a function by a positive constant, and is preserved by function addition, we conclude that this property applies to the entire function $M_j(t_1,t_2,t_3,t_4)$. Namely, for any vector $u$, the second derivative of $M_j(t_1,t_2,t_3,t_4)$ in the direction $u$ is at most $12M_j(t_1,t_2,t_3,t_4)$. Since $M_j(t_1,t_2,t_3,t_4)$ has value 1 at the origin, and 0 gradient, we can bound its value at any multiple $x u$ by solving the differential equation $f(0)=1, f'(0)=0, f''(x)\leq 12 f(x)$, to yield $f(x)\leq \cosh(\sqrt{12}x)\leq e^{(12/2)x^2}$. Since this bounds $M(t_1,t_2,t_3,t_4)$ when $(t_1,t_2,t_3,t_4)=xu$ for any unit vector $u$, we conclude that $M(t_1,t_2,t_3,t_4)\leq e^{6\sum_{i=1}^4 t_i^2}$.
	
	Moving on to the next case, if $x_j=0,y_j=1$ case, then the 4-tuple of pairs $(x_j+x'_j,y_j+y'_j),(x'_j,y'_j),(x_j+x'_j,y'_j),(x'_j,y_j+y'_j)$ contains 2 copies of $(x'_j,y'_j)$ and 2 copies of $(x'_j,1+y'_j)$. Thus the contribution to the counts $d_1,d_2,d_3,d_4$ will be uniformly chosen among the 4-tuples $(2,2,0,0),(0,0,2,2)$, where there are only 2 possibilities in this case. The moment generating function of this uniform distribution over 2 possibilities is just the sum of the 2 terms in the definition of the MGF, which as usual we center at the mean, $(1,1,1,1)$:
	\[M_j(t_1,t_2,t_3,t_4)=\frac{1}{2}e^{(2-1)t_1+(2-1)t_2+(0-1)t_3+(0-1)t_4}+\frac{1}{2}e^{(0-1)t_1+(0-1)t_2+(2-1)t_3+(2-1)t_4}\]
	This equals $\cosh(\langle (t_1,t_2,t_3,t_4), (1,1,-1,-1)\rangle) \leq e^{\langle (t_1,t_2,t_3,t_4), (1,1,-1,-1)\rangle^2/2}\leq e^{(||(1,1,-1,-1)||_2^2/2)\sum_{i=1}^4 t_i^2}$, namely $e^{2\sum_{i=1}^4 t_i^2}$.
	
	The next case, if $x_j=1,y_j=0$, is analogous to the previous case. Here the 4-tuple of pairs $(x_j+x'_j,y_j+y'_j),(x'_j,y'_j),(x_j+x'_j,y'_j),(x'_j,y_j+y'_j)$ contains 2 copies of $(x'_j,y'_j)$ and 2 copies of $(1+x'_j,y'_j)$. Thus the contribution to the counts $d_1,d_2,d_3,d_4$ will be uniformly chosen among the 4-tuples $(2,0,2,0),(0,2,0,2)$. The moment generating function of this uniform distribution over 2 possibilities is just the sum of the 2 terms in the definition of the MGF, which as usual we center at the mean, $(1,1,1,1)$:
	\[M_j(t_1,t_2,t_3,t_4)=\frac{1}{2}e^{(2-1)t_1+(0-1)t_2+(2-1)t_3+(0-1)t_4}+\frac{1}{2}e^{(0-1)t_1+(2-1)t_2+(0-1)t_3+(2-1)t_4}\]
	Analogously to the previous case, this is at most $e^{2\sum_{i=1}^4 t_i^2}$.
	
	Finally, in the case $x_j=1,y_j=1$, then the 4-tuple of pairs $(x_j+x'_j,y_j+y'_j),(x'_j,y'_j),(x_j+x'_j,y'_j),(x'_j,y_j+y'_j)$ equals $(1+x'_j,1+y'_j),(x'_j,y'_j),(1+x'_j,y'_j),(x'_j,1+y'_j)$ and thus always contributes $(1,1,1,1)$ to the 4 counts. Thus the moment generating function centered at $(1,1,1,1)$ is, trivially, $M_j(t_1,t_2,t_3,t_4)=1$.
	
	Putting together the pieces: since the moment generating function is simply the product of its contribution from each index $j$ from 1 to $d$, and in all cases the we bounded this contribution by $e^{6\sum_{i=1}^4 t_i^2}$, we have that the overall MGF of all $d$ indices is bounded as \[M(t_1,t_2,t_3,t_4)\leq e^{6d\sum_{i=1}^4 t_i^2}\] We now bound the desired quantity in this lemma via \eqref{eq:moment-gaussian}. We thus have \begin{align*}\Exp_{(t_1,t_2,t_3,t_4)\leftarrow \mathcal{N}(0,2s)}[M(t_1,t_2,t_3,t_4)]&\leq \Exp_{(t_1,t_2,t_3,t_4)\leftarrow \mathcal{N}(0,2s)}[e^{6d\sum_{i=1}^4 t_i^2}]\\
		&= \int_{\mathbb{R}^4} \frac{1}{\sqrt{2\pi\cdot 2s}^4}e^{-\frac{1}{4s}\sum_{i=1}^4 t_i^2}e^{6d\sum_{i=1}^4 t_i^2}\,dt_1\,dt_2\,dt_3\,dt_4\\
		&=\sqrt{\frac{\frac{1}{1/(4s)-6d}}{2\cdot 2s}}^4 = \left(\frac{1}{1-24ds}\right)^2
	\end{align*}
	as desired.
\end{proof}

The following proposition establishes an $L_2$ upper bound on the distribution induced by the randomization process in our reduction.

\begin{proposition}\label{prop:combi}
	For any dimension $d$ and any vectors $\bx,\by\in\{0,1\}^d$, let $D=100 d$. Pick random vectors $\bx',\by'\leftarrow \{0,1\}^d$; pick random vectors $\bx'',\by''\leftarrow \{0,1\}^D$ such that there are an even number of indices $j\in\{1,\ldots,D\}$ where $\bx''_j=y''_j=1$; and pick a random permutation $\tau$ of $4d+D$ elements. The claim is that the pair of length $4d+D$ vectors 
	\begin{align*}
		X\coloneqq\tau(\bx+\bx',\bx',\bx+\bx',\bx',\bx''),\\Y\coloneqq\tau(\by+\by',\by',\by',\by+\by',\by''),
	\end{align*}
	considered as a distribution $\Dcal$ over domain$\{0,1\}^{2(4d+D)}$, has $L_2$ norm at most $8\cdot 2^{-(4d+D)}$, which is $8$ times the $L_2$ norm of the uniform distribution over this domain.
\end{proposition}

\begin{proof}
	Among the $4d$ entries of the pair $A=(\bx+\bx',\bx',\bx+\bx',\bx'),B=(\by+\by',\by',\by',\by+\by')$ let $d_1$ count the number of indices where $A_j=B_j=0$; let $d_2$ count the entries where $A_j=0,B_j=1$; let $d_3$ count the entries where $A_j=1,B_j=0$; and let $d_4$ count the entries where $A_j=B_j=1$. Thus $d_1+d_2+d_3+d_4=4d$.
	
	Since $\Dcal= \sum_{d_1,d_2,d_3,d_4} \Pr_{\Dcal}[d_1,d_2,d_3,d_4] \Dcal_{|d_1,d_2,d_3,d_4}$, by the triangle inequality we have \begin{equation}\label{eq:triangle}||\Dcal||_2\leq \sum_{d_1,d_2,d_3,d_4} \Pr_{\Dcal}[d_1,d_2,d_3,d_4] ||\Dcal_{|d_1,d_2,d_3,d_4}||_2\end{equation}

	For the random vectors $\bx'',\by''$, let $D_1,D_2,D_3,D_4$ count the number of indices $j\in\{1,\ldots,D\}$ respectively where $x''_j=0,y''_j=0$; where $x''_j=0,y''_j=1$; where $x''_j=1,y''_j=0$; and where $x''_j=1,y''_j=1$. Thus $D_1+D_2+D_3+D_4=D$.
	
	Thus there will be $D_1+d_1$ total indices where $X_j=0,Y_j=0$, etc. Given these total counts $D_1+d_1, D_2+d_2,D_3+d_3,D_4+d_4$, the total number of rearrangements of these columns is the multinomial ${D+4d\choose D_1+d_1, D_2+d_2,D_3+d_3,D_4+d_4}$; and the random permutation $\tau$ will choose a uniformly random one of these arrangements.
	
	The probability of the random vectors $\bx'',\by''$ having counts exactly $D_1,D_2,D_3,D_4$ would be exactly $4^{-D}{D\choose D_1,D_2,D_3,D_4}$ if we were not conditioning on $D_4$ being even. The probability of $D_4$ being even is $\geq \frac{1}{2}$, since the difference between the number of instantiations where $D_4$ is even versus $D_4$ is odd can be exactly computed as $\sum_{(D_1, D_2,D_3,D_4): \sum_i D_i = D}{D \choose D_1,D_2,D_3,D_4 } (-1)^{D_4}=(1+1+1-1)^D\geq 0$. Thus, by the law of conditional probability, we have
	\begin{align*}
		&4^{-D}{D\choose D_1,D_2,D_3,D_4} = \pr\pcc{D_1,D_2,D_3,D_4}\\
		=& \pr\pcc{D_1,D_2,D_3,D_4|D_4\text{ is even}}\pr\pcc{D_4\text{ is even}} + \pr\pcc{D_1,D_2,D_3,D_4|D_4\text{ is odd}}\pr\pcc{D_4\text{ is odd}}\\
		\ge&\frac{1}{2}\pr\pcc{D_1,D_2,D_3,D_4|D_4\text{ is even}},
	\end{align*}
	therefore the probability of $D_1,D_2,D_3,D_4$ with even $D_4$ is $\leq 2\cdot 4^{-D}{D\choose D_1,D_2,D_3,D_4}$.
	
	Thus, overall, the random process $\Dcal_{|d_1,d_2,d_3,d_4}$ can be described as picking $D_1,D_2,D_3,D_4$ with probability $\leq 2\cdot 4^{-D}{D\choose D_1,D_2,D_3,D_4}$, and then uniformly splitting this probability among the ${D+4d\choose D_1+d_1, D_2+d_2,D_3+d_3,D_4+d_4}$ possible rearrangements of these columns. We point out that, for a distribution where, for different indices $j$, we have $p_j$ probability mass uniformly divided among $n_j$ elements, its squared $L_2$ norm is $\sum_j n_j (\frac{p_j}{n_j})^2=\sum_j \frac{p_j^2}{n_j}$. Thus bound the $L_2$ norm of our conditional distribution as
	\[
	||\Dcal_{|d_1,d_2,d_3,d_4}||_2\leq \sqrt{\sum_{(D_1, D_2,D_3,D_4): \sum_i D_i = D}\frac{\left(2\cdot 4^{-D}{D\choose D_1,D_2,D_3,D_4}\right)^2}{{D+4d\choose D_1+d_1, D_2+d_2,D_3+d_3,D_4+d_4}}}
	\]
	This expression is exactly $2\cdot 4^{-D}$ times the square root of the expression bounded in Lemma~\ref{lem:binom-square-ratio}, if in Lemma~\ref{lem:binom-square-ratio} we reparameterize $d$ as $4d$. Namely, \[||\Dcal_{|d_1,d_2,d_3,d_4}||_2\leq 2\cdot e^{\frac{2}{D}\sum_{i=1}^4 (d_i-d)^2}\cdot \p{1+\frac{4d}{D}}^{3/4}\cdot 2^{-(4d+D)}\]
	
	Combining this bound with \eqref{eq:triangle} (the triangle inequality), we have \begin{align*}||\Dcal||_2&\leq \Exp_{d_1,d_2,d_3,d_4} [||\Dcal_{|d_1,d_2,d_3,d_4}||_2]\leq 2\p{1+\frac{4d}{D}}^{3/4}\cdot 2^{-(4d+D)}\Exp_{d_1,d_2,d_3,d_4} [e^{\frac{2}{D}\sum_{i=1}^4 (d_i-d)^2}]\end{align*}
	
	The right hand side is exactly the form of Lemma~\ref{lem:xor-identity}, applied with $s=\frac{2}{D}$, thus yielding our overall bound of
	\[||\Dcal||_2\leq2\p{1+\frac{4d}{D}}^{3/4}\cdot 2^{-(4d+D)}\left(\frac{1}{1-48\frac{d}{D}}\right)^2\]
	
	For $D\geq 100d$, we see that $||\Dcal||_2\leq 8\cdot 2^{-(4d+D)}$, as desired.
\end{proof}

The following lemma guarantees that activation functions which satisfy \asmref{asm:lb_relaxed}, can be simulated to arbitrary accuracy, using depth 2 threshold networks of polynomial width.

\begin{lemma}\label{lem:thresh_approx}
	Let $\sigma$ be an activation function that satisfies \asmref{asm:lb_relaxed}. Then for all $\delta>0$, there exists a depth 2 threshold network $\Ncal$ of width at most $\frac{\poly(R)}{\delta}$ and weights of magnitude at most $\frac{\poly(R)}{\delta}$, such that
	\[
	\sup_{x\in[-R,R]}\abs{\Ncal(x) - \sigma(x)} \le \delta.
	\]
\end{lemma}

\begin{proof}
	We will first construct a piecewise constant function that approximates $\sigma$ to an $L_{\infty}$ distance of $\delta$, and then we will compute this piecewise constant function precisely using a depth 2 threshold network.
	
	Let $x_1\coloneqq-R$ and $y_1\coloneqq\sigma(x_1)$. We define the set 
	\[
	A_1\coloneqq\{a\in[-R,R]:\forall x\le a,\hskip 0.15cm |\sigma(x)-y_1|\le\delta\}.
	\]
	Note that $A_1\neq\emptyset$ since $x_1\in A_1$, and define $x_1'=x_1$ and $x_2\coloneqq \sup A_1$.
	We now split our analysis into three cases, depending on the continuity properties of $\sigma$ at $x_2$.
	\begin{enumerate}
		\item \label{item:case1}
		Suppose that $\sigma$ is continuous at $x_2$ from the right. By the definition of $A_1$ we have that $x_1<x_2$, since if $x_1=x_2$ by contradiction, due to right continuity, we can find an interval $[x_1,b)$ for $b>x_1$ close enough to $x_1$, such that $|\sigma(x)-y_1|\le\delta$ for all $x\in[x_1,b)$, which contradict the definition of $A_1$. We define $y_2\coloneqq \sigma(x_2)$, and we have by the definition of $A_1$ that $\sup_{x\in[x_1,x_2)}|\sigma(x)-y_1|\le\delta$. Moreover, if $\sigma$ is continuous at $x_2$ from both sides then this implies that $|y_2-y_1|=\delta$, and if it is only continuous from the right this implies that $|y_2-y_1|\ge\delta$. This can be seen to hold true since if we had $|y_2-y_1|<\delta$ by contradiction, then from right continuity there exists an interval $[x_2,b)$ for some $b>x_2$ such that $|\sigma(x)-y_1|<\delta$ for all $x\in [x_2,b)$, which contradicts the definition of $A_1$. We can now define the set
		\[
		A_2\coloneqq\{a\in[x_2,R]:\forall x\in[x_2,a],\hskip 0.15cm |\sigma(x)-y_2|\le\delta\},
		\]
		and note that $A_2\neq\emptyset$ since $x_2\in A_2$. Define $x_2'=x_2$, we can conclude that in this case, we have an interval $[x_1,x_2)$ such that $|\sigma(x)-y_1|\le\delta$ for all $x\in [x_1,x_2)$, that $|\sigma(x_2)-y_2|=0\le\delta$, and that $|\sigma(x_1')-\sigma(x_2')|\ge\delta$.
		
		\item \label{item:case2}
		Suppose that $\sigma$ is not continuous at $x_2$ from the right, but it is either continuous at $x_2$ from the left or $x_1=x_2$, which in both cases implies that $|\sigma(x_2)-y_1|\le\delta$. It must also hold that $|\sigma(x)-y_1|\le\delta$ for all $x\in[x_1,x_2]$. Define $y_2\coloneqq\lim_{x\to x_{2^+}}\sigma(x)$, note that by the definition of $A_1$ and $y_2$, we have similarly to the previous case that $|y_2-y_1|\ge\delta$, and let $x_2'$ be close enough to $x_2$ from the right so that from right continuity we get $|y_2-\sigma(x_2')|\le0.25\delta$, and thus $|\sigma(x_1')-\sigma(x_2')|\ge0.75\delta$. We can now define the set
		\[
		A_2\coloneqq\{a\in(x_2,R]:\forall x\in(x_2,a],\hskip 0.15cm |\sigma(x)-y_2|\le\delta\},
		\]
		which is not empty since $(x_2,x_2']\subseteq A_2$.
		\item 
		Suppose that $x_1<x_2$ and that $\sigma$ has a discontinuity at $x_2$ from both sides, and note that together with the previous two items, this covers all possibilities. Then, since $\sigma$ is of bounded variation, we have that both limits at the two sides exist and are finite. Next, if $x_2\in A_1$, then this implies that $|\sigma(x_2)-y_1|\le\delta$. We can now define $y_2\coloneqq\lim_{x\to x_{2^+}}\sigma(x)$ and proceed in the same manner as in \itemref{item:case2} to define $A_2$ and $x_2'
		$. Otherwise, we have that $x_2\notin A_1$, which implies that $|\sigma(x_2)-y_1|>\delta$. We define $y_2\coloneqq\sigma(x_2)$ which trivially implies $|\sigma(x_2)-y_2|=0\le\delta$, and proceed to define $x_2'$ and $A_2$ as in \itemref{item:case1}.
	\end{enumerate}
	Define $x_3\coloneqq\sup A_2$, we can continue in this manner and define sequences of target values $y_1,y_2,\ldots$ and points $x_1,x_2,\ldots$, $x_1',x_2',\ldots$ such that for all $i$, we have 
	\begin{itemize}
		\item 
		$|\sigma(x)-y_i|\le\delta$ for all $x\in(x_i,x_{i+1})$.
		\item 
		$|\sigma(x_i)-y_{i-1}|\le\delta$ or $|\sigma(x_i)-y_i|\le\delta$.
		\item 
		$|\sigma(x_i)-\sigma(x_{i+1})|\ge\delta$ and $|\sigma(x_i)-\sigma(x_i')|\le0.25\delta$, which imply $|\sigma(x_i')-\sigma(x_{i+1}')|\ge0.5\delta$.
	\end{itemize}
	We now bound the length $n$ of the sequences required to get $R\in A_n$; namely, the number of piecewise constant segments required to approximate $\sigma$ to accuracy $\delta$ uniformly on $[-R,R]$. We have by \asmref{asm:lb_relaxed} that $\sigma$ has total variation at most $C_{\sigma}(1+2R)^{\alpha_{\sigma}}$. On the other hand, we have from the above properties that the partition $x_1',x_2',\ldots,x_n'$ of $[-R,R]$ has total variation at least
	\[
	\sum_{i=1}^{n-1}\abs{\sigma(x_{i+1}')-\sigma(x_i')} \ge 0.5(n-1)\delta.
	\]
	Combining these two inequalities, we obtain $0.5(n-1)\delta \le C_{\sigma}(1+2R)^{\alpha_{\sigma}}$, implying $n=\frac{\poly(R)}{\delta}$.
	
	It now only remains to approximate a piecewise linear function, with $\frac{\poly(R)}{\delta}$ constant segments, using a threshold network with a similar number of neurons. Define $y_0=0$, it is easy to verify that this approximation is given by the expression
	\[
	\sum_{i=1}^n \xi_i(y_i-y_{i-1})\sigma_{\thresh}(\xi_i(x-x_i)-0.5) - 0.5(\xi_i-1)(y_i-y_{i-1}),
	\]
	where $\xi_i\in\{-1,1\}$ is chosen according to the discontinuity of our piecewise constant function at the point $x_i$. Namely, by setting $w_i\coloneqq\xi_i$, $b_i\coloneqq-\xi_ix_i-0.5$, $v_i\coloneqq\xi_i(y_i-y_{i-1})$ and $b_0\coloneqq-0.5\sum_{i=1}^n(\xi_i-1)(y_i-y_{i-1})$, we obtain a width $n$ threshold network
	\[
	\Ncal(x)\coloneqq\sum_{i=1}^nv_i\sigma_{\thresh}(w_ix+b_i)+b_0,
	\]
	satisfying
	\[
	\sup_{x\in[-R,R]}\abs{\sigma(x)-\Ncal(x)}\le\delta.
	\]
	Lastly, by \asmref{asm:lb_relaxed}, we have that $|y_i|\le\poly(R)$ for all $i\in[n]$, implying that $\Ncal$ has weights of magnitude at most $\frac{\poly(R)}{\delta}$.
\end{proof}

The following proposition guarantees that functions on the Boolean hypercube, computed by depth 2 networks, which employ activation functions that satisfy \asmref{asm:lb_relaxed}, can be simulated to arbitrary accuracy using depth 2 threshold networks of width polynomial in the size of the network.

\begin{proposition}\label{prop:thresh_approx}
	Let $\delta>0$, suppose that $\sigma$ satisfies \asmref{asm:lb_relaxed}, and let $f:\{0,1\}^d\to\reals$ be a function computed by a depth 2, width $m$ $\sigma$-network, with weights bounded by $C$. Then there exists a depth 2 threshold network $\Ncal$ of width $\frac{m^2\poly(d,C)}{\delta}$ and weights bounded in magnitude by $\frac{\poly(d,C)}{\delta}$, such that
	\[
	\max_{\bx\in\{0,1\}^d} \abs{f(\bx)-\Ncal(\bx)} \le \delta.
	\]
\end{proposition}

\begin{proof}
	Let
	\[
	f(\bx) \coloneqq \sum_{i=1}^m v_i \sigma\p{\inner{\bw_i,\bx} + b_i} + b_0
	\]
	be the function computed by the network $f$. We first observe that by our weight boundedness assumption, we have for all $i\in[m]$ that $|\inner{\bw_i,\bx} + b_i| \le \norm{\bw_i}\norm{\bx} + C \le (d+1)C$. We now use \lemref{lem:thresh_approx} to approximate each $\sigma$-neuron in $f$ to accuracy $\frac{\delta}{mC}$, and obtain $m$ depth 2, width $\frac{m\poly(d,C)}{\delta}$ threshold networks $\Ncal_1,\Ncal_2,\ldots,\Ncal_m$, such that
	\[
	\abs{\Ncal_i(\bx)-\sigma\p{\inner{\bw_i,\bx} + b_i}} \le \frac{\delta}{mC}
	\]
	for all $i\in[m]$. Define the network $\Ncal$ given by
	\[
	\Ncal(\bx)\coloneqq \sum_{i=1}^m v_i\Ncal_i(\bx) + b_0,
	\]
	fix any $\bx\in\{0,1\}^d$ and compute
	\[
		\abs{f(\bx) - \Ncal(\bx)} \le \sum_{i=1}^m 	v_i\abs{\sigma\p{\inner{\bw_i,\bx} + b_i} - \Ncal_i(\bx)} \le \sum_{i=1}^m C \frac{\delta}{mC} = \delta.
	\]
	Lastly, since $\Ncal$ is a linear combination of depth 2 neural networks, each of which is of width at most $\frac{m\poly(d,C)}{\delta}$ and weights of magnitude $\frac{\poly(d,C)}{\delta}$, we have that $\Ncal$ is in itself a depth 2, width $\frac{m^2\poly(d,C)}{\delta}$ threshold network, with weights bounded by $\frac{\poly(d,C)}{\delta}$ as required.
\end{proof}

The following theorem, which may be of independent interest, establishes a reduction for hard-to-approximate functions on the Boolean hypercube to the unit ball.

\begin{theorem}\label{thm:ball_reduction}
    Suppose that $\sigma$ satisfies \asmref{asm:lb_relaxed}, let $\beta\coloneqq\max\{1,\alpha_{\sigma}\}$, let $\bx_1,\ldots,\bx_{2^{2d}}\in\{0,1\}^{2d}$ be an enumeration of the Boolean hypercube, and let $\Fcal_{d,C,\beta}$ denote the function class of depth 2 $\sigma$-networks $\reals^d\to\reals$ with weights bounded by $C\le 2^{\frac{d}{80\beta+80}}$ and with width at most $2^{\frac{d}{80\beta}}$. Further assume that $g:\{0,1\}^{2d}\to\{0,1\}$ is a Boolean function satisfying
    \[
        \inf_{\Ncal\in\Fcal_{2d,3dC,\beta}} \Exp_{\bx\sim \Ucal\p{\set{0,1}^{2d}}}\pcc{\p{\Ncal(\bx) - g(\bx)}^2} > c+\varepsilon,
    \]
    for some constant $c>0$ and $\varepsilon>0$. Then, there exists a set $\bz_1,\ldots,\bz_{2^{2d}}\in\reals^{2d}$ such that the set $\Scal_d\coloneqq\p{\bigcup_{i=1}^{2^{2d}}\set{(\bz_i,\frac{1}{4\sqrt{d}}\bx_i)}}$ satisfies $\Scal_d\subseteq B_{0.9}(\mathbf{0})$, every pair of points in $\Scal_d$ is at distance at least $0.4$ apart, and it holds that
	\begin{equation*}
		\inf_{\Ncal\in\Fcal_{4d,C,\beta}} \Exp_{\p{\bz,\frac{1}{4\sqrt{d}}\bx}\sim \Ucal\p{\Scal_d}}\pcc{\p{\Ncal\p{\bz,\frac{1}{4\sqrt{d}}\bx} - g(\bx)}^2} > c.
	\end{equation*}
\end{theorem}

\begin{proof}
    Let $\bz_1,\ldots,\bz_{2^{2d}}\in B_{0.8}^{2d}(\mathbf{0})$ be a set of points such that $\min_{i\neq j}\norm{\bz_i-\bz_j}_2>0.4$. To show that such a set exists, we use \citet[Eq.~(2)]{jenssen2018kissing} which guarantees that a set $\bz'_1,\bz'_2,\ldots$ such that $\norm{\bz'_i}=1$ for all $i$, which also satisfies $\inner{\bz'_i,\bz'_j}\le\cos(\theta)$ for all $i\neq j$ and some $\theta$, can be of size at least
    \[
        (1+o(1))\sqrt{2\pi d}\cdot\frac{\cos{\theta}}{\sin^{d-1}(\theta)}.
    \]
    Plugging $\theta=\pi/6$ in the above, it is at least $2^d$ for all $d\ge1$. Moreover, we have
    \[
        \norm{\bz'_i-\bz'_j}^2=\norm{\bz'_i}^2+\norm{\bz'_j}^2 - 2\inner{\bz'_i,\bz'_j} \ge 2 - 2\cos(\pi/6) \ge 0.25,
    \]
    where taking the square root of the above inequality and scaling all $\bz'_i$ in $2d$ dimensions by a factor of $0.8$ gives the set $\bz_1,\ldots,\bz_{2^{2d}}$. 
    
    Next, we have that $\norm{(\bz_i,0.08\bx_j)}\le\sqrt{0.8^2+2\cdot0.08^2}\le 0.9$ for all $i$ and $j$, and that the minimal distance between any two points $(\bz_{i_1},0.08\bx_{j_1})$ and $(\bz_{i_2},0.08\bx_{j_2})$ for $i_1\neq i_2$ is at least $0.4$.

    For a neural network $\Ncal$, define the loss function $\ell_{\Ncal}\p{\bz,0.08\bx}\coloneqq \p{\Ncal\p{\bz,0.08\bx} - g(\bx)}^2$. We now bound the loss over all instances in our domain and neural networks in our function class as follows. Starting with the output of a single neuron, we have by \asmref{asm:lb_relaxed} that for any weight vector $\bw$, bias term $b$, and input $\bx\in\reals^{4d}$ such that $\max_i|x_i|\le1$, that the following upper bound holds
    \[
        \abs{\sigma(\inner{\bw,\bx}+b)} \le C_{\sigma}\p{1+|\norm{\bw}_{\infty}\cdot\norm{\bx}_1 + C|^{\alpha_{\sigma}}} = \Ocal(C^{\alpha_{\sigma}}d^{\alpha_{\sigma}}).
    \]
    So the output of any function in $\Fcal_{4d,C,\beta}$ is at most $\Ocal(w C^{\alpha_{\sigma}+1} d^{\alpha_{\sigma}})$, and therefore
    \begin{equation}\label{eq:ell_bound}
        \ell_{\Ncal}(\bz,0.08\bx) = \Ocal(w^2 C^{2\alpha_{\sigma}+2} d^{2\alpha_{\sigma}})
    \end{equation}
    for all $\Ncal\in\Fcal_{4d,C,\beta}$, $\bz\in B_{0.8}^{2d}$ and $\bx\in\{0,1\}^{2d}$.
    

    We define a sample size $n=2^{2d/3}$, and consider the process of drawing a set of size $n$ comprised of instances of the form $(\bz_i,\bx_j)$, where $i,j\sim \Ucal([2^{2d}])$ are independent. Denote $\Ncal_S\coloneqq \argmin{\Ncal\in\Fcal_{4d,C,\beta}}\Lcal_S(\Ncal)$ and $\Ncal^*\coloneqq \argmin{\Ncal\in\Fcal_{4d,C,\beta}}\Lcal_{\Dcal}(\Ncal)$, where $\Lcal_S$ and $\Lcal_{\Dcal}$ are the empirical loss and risk, respectively, over the distribution $\Dcal$ defined by our process. We now have that
    \[
        \Lcal_{\Dcal}(\Ncal^*) - \Lcal_{S}(\Ncal_S) = \Lcal_{\Dcal}(\Ncal^*) - \Lcal_{\Dcal}(\Ncal_S) + \Lcal_{\Dcal}(\Ncal_S) - \Lcal_{S}(\Ncal_S) \le \Lcal_{\Dcal}(\Ncal_S) - \Lcal_{S}(\Ncal_S),
    \]
    where the inequality follows from the definition of $\Ncal^*$.
    By \eqref{eq:ell_bound} and a standard Rademacher complexity argument, we have with probability at least $1-\delta$ that the above is at most
    \begin{equation}\label{eq:uniform_convergence}
        2\Rcal_n(\Fcal_{4d,C,\beta}) + 4\Ocal(w^2 C^{2\alpha_{\sigma}+2} d^{2\alpha_{\sigma}})\sqrt{\frac{2\ln(2/\delta)}{n}}
    \end{equation}
    \citep[Thm.~26.5]{shalev2014understanding}. Recall that for all $\Ncal\in\Fcal_{4d,C,\beta}$ we have $w\le 2^{\frac{d}{20\beta}} \le 2^{\frac{d}{20}}$ and $C\le 2^{\frac{d}{20\beta+20}} \le 2^{\frac{d}{20\alpha_{\sigma}+20}}$, we can thus upper bound the second summand as follows
    \begin{equation}\label{eq:weight_bound}
        \Ocal(w^2 C^{2\alpha_{\sigma}+2} d^{2\alpha_{\sigma}})\sqrt{\frac{2\ln(2/\delta)}{n}} = \Ocal\p{2^{0.1d}\cdot 2^{0.1d}d^{2\alpha_{\sigma}}}\sqrt{\frac{\ln(2/\delta)}{2^{2d/3}}} = \Ocal\p{\exp(-\Omega(d))}\sqrt{\ln(2/\delta)}.
    \end{equation}
    Using \asmref{asm:lb_relaxed} and the facts that $w\le 2^{\frac{d}{20\beta}} \le 2^{\frac{d}{20\alpha_{\sigma}}}$ and $C\le 2^{\frac{d}{20\beta+20}} \le 2^{\frac{d}{20\alpha_{\sigma}+20}} \le 2^{\frac{d}{20\alpha_{\sigma}}}$, the Rademacher complexity term can also be upper bounded,
    \[
        \Rcal_n(\Fcal_{4d,C,\beta}) \le \frac{\p{w\cdot C\cdot 4d}^{\alpha_{\sigma}}}{\sqrt{n}} \le 2^{-d/3} \cdot 2^{0.1d}(4d)^{\alpha_{\sigma}} = \Ocal\p{\exp(-\Omega(d))}.
    \]
    Plugging the above and \eqref{eq:weight_bound} in \eqref{eq:uniform_convergence}, rearranging, and using the fact that $\delta\le1$, we arrive at
    \[
        \Lcal_{\Dcal}(\Ncal^*) \le \Lcal_{S}(\Ncal_S) + \Ocal\p{\exp(-\Omega(d))}\sqrt{\ln(2/\delta)}.
    \]
    Conditioning the above bound on the event where the sample $S$ consists of distinct choices of $i$ and $j$, which happens with probability at least $1-\Omega(2^{-2d/3})\ge0.5$, we have that the following bound holds for such distinct choices with probability at least $1-
    \delta$
    \[
        \Lcal_{\Dcal}(\Ncal^*) \le \Lcal_{S}(\Ncal_S) + \Ocal\p{\exp(-\Omega(d))}\sqrt{\ln(4/\delta)}.
    \]
    We now describe a different process for sampling sets $S$ that is equivalent to the above conditional distribution. Sample a random permutation of the numbers $[2^{2d}]$, and sample a random permutation of the $2^{2d}$ points in the set $\{0,1\}^{2d}$. Next, pair up integers and points; and then let $P_j$ be the $j^\textrm{th}$ such point. We now produce $2^{4d/3}$ different samples $S_j$ of size $2^{2d/3}$ by defining the $k^{\textrm{th}}$ sample to consist of points $P_{(k-1)2^{2d/3}:k2^{2d/3}-1}$. For any given hypothesis function $\Ncal$, the average error over points in $P$ equals the average error over these $2^{4d/3}$ sample sets of the average error in the sample; and thus the minimum error over any $\Ncal\in \Fcal_{4d,C,\beta}$ for $P$ is at least the average over these $2^{4d/3}$ sample sets of the minimum error over that sample set. Denoting the resulting concatenated set of $2^{4d/3}$ samples by $\Scal_d$, and its empirical minimizer by $\Ncal_{\Scal}$, we have that
    \[
        \frac{1}{2^{4d/3}}\sum_{j=1}^{2^{4d/3}}\Lcal_{S_j}(\Ncal_{S_j}) \le \Lcal_{\Scal_d}(\Ncal_{\Scal_d}).
    \]
    Thus, using a union bound to bound the probability that any of these $2^{4d/3}$ sample sets is in the ``bad" fraction of sample sets where $i$ and $j$ are not distinct, we have with probability at least $1-\delta$ that
    \begin{equation}\label{eq:exp_union_bound}
        \min_{\Ncal\in\Fcal_{4d,C,\beta}} \Exp_{i,j\sim \Ucal([2^{2d}])}\pcc{\p{\Ncal(\bz_i,\bx_j) - g(\bx_j)}^2} = \Lcal_{\Dcal}(\Ncal^*) \le \Lcal_{\Scal_d}(\Ncal_{\Scal_d}) + \Ocal\p{\exp(-\Omega(d))}\sqrt{\ln\p{\frac{2^{4d/3+2}}{\delta}}}.
    \end{equation}
    Focusing on the left hand side first, we observe that since $i$ and $j$ are independent, there must exist some $i'\in[2^{2d}]$ for which the expectation term is smaller than the above, namely
    \[
        \min_{\Ncal\in\Fcal_{4d,C,\beta}} \Exp_{j\sim \Ucal([2^{2d}])}\pcc{\p{\Ncal(\bz_{i'},\bx_j) - g(\bx_j)}^2} \le \min_{\Ncal\in\Fcal_{4d,C,\beta}} \Exp_{i,j\sim \Ucal([2^{2d}])}\pcc{\p{\Ncal(\bz_i,\bx_j) - g(\bx_j)}^2}.
    \]
    However, since $\bz_{i'}$ is fixed, for any $\Ncal\in\Fcal_{4d,C,\beta}$, we can construct a network $\Ncal'\in\Fcal_{2d,3dC,\beta}$ such that $\Ncal(\bz_{i'},\bx_j)=\Ncal'(\bx_j)$. This holds true since the first hidden layer of $\Ncal'$ can absorb the linear combination obtained by multiplying $\bz_{i'}$ with the corresponding weight vectors into the bias terms, which will increase their magnitudes by at most $2dC$ for a total less than $3dC$. Thus, we have
    \[
        \min_{\Ncal\in\Fcal_{2d,3dC,\beta}} \Exp_{j\sim \Ucal([2^{2d}])}\pcc{\p{\Ncal(\bx_j) - g(\bx_j)}^2} \le \min_{\Ncal\in\Fcal_{4d,C,\beta}} \Exp_{j\sim \Ucal([2^{2d}])}\pcc{\p{\Ncal(\bz_{i'},\bx_j) - g(\bx_j)}^2}.
    \]
    By our theorem assumption, the above left hand side is at least $c+\varepsilon$, so by plugging this back in \eqref{eq:exp_union_bound} we arrive at
    \[
        c+\varepsilon < \Lcal_{\Scal_d}(\Ncal_{\Scal_d}) + \Ocal\p{\exp(-\Omega(d))}\sqrt{\ln\p{\frac{2^{4d/3+2}}{\delta}}}.
    \]
    Finally, taking $\delta=0.5$ for example,\footnote{We remark that this argument will work even for $\delta$ which decays doubly exponentially to zero with $d$ since we have a logarithmic term multiplying an exponentially decaying term, which indicates that all matchings between $\bz_i$'s and $\bx_j$'s, except for a minuscule fraction, satisfy our requirements.} we have by the probabilistic method that there exists some realization of the matchings between $\bz_i$'s and $\bx_j$'s, and hence a set $\Scal_d$, such that for sufficiently large $d$, it holds that
    \[
        c+\varepsilon < \Lcal_{\Scal_d}(\Ncal_{\Scal_d}) + \varepsilon,
    \]
    from which the theorem follows.
\end{proof}

With the above auxiliary results at hand, we are finally ready to prove the theorem.

\begin{proof}[Proof of \thmref{thm:lb}]\label{app:lb_proof}
    First, for any natural $d$ and real $C>0$, define $w(d,C)$ to be the minimal width required for a depth 2 $\sigma$-network with weights bounded in magnitude by $C$ to approximate $f_d$ to accuracy at most $\frac{1}{400}$. Our goal is therefore to derive a lower bound on $w(d,C)$.
	
    Let $D\coloneqq 100d$, and let $\Scal_{4D+16d}\subseteq B_{0.9}^{4D+16d}(\mathbf{0})$ be some set to be defined later, with the property that any two points in this set are at a distance of at least $0.4$ apart (essentially, this will be the set whose existence is established by \thmref{thm:ball_reduction}). Recall the shorthand $\Ccal_{4d}\coloneqq\pcc{0,\frac{1}{12\sqrt{d}}}^{4d}$, we define
    \[
        \Acal_{4D+16d}\coloneqq \Scal_{4D+16d}+\Ccal_{4D+16d}.
    \]
    Observe that due to $\norm{\bz}\le0.9$ for all $\bz\in\Scal_{4D+16d}$ and $\norm{\bc}\le1/36$ for all $\bc\in\Ccal_{4D+16d}$, it holds that $\norm{\bx}\le1$ for all $\bx\in\Acal_{4D+16d}$. Moreover, the minimal distance between any two points in $\Acal_{4D+16d}$ is at least the minimal distance between two points in $\Scal_{4D+16d}$ minus the diameter of $\Ccal_{4D+16d}$, i.e.\ $1/36$. Thus, the minimal distance between any two points in $\Acal_{4D+16d}$ is constant.
    
    We first assume that we are given a depth 2, width $w(D+4d,C)$ $\sigma$-network $\Ncal$, with weights bounded by $C$, that satisfies
	\[                                 
            \Exp_{(\bx,\by)\sim\Ucal(\Acal_{4D+16d})}\pcc{\p{\Ncal(\bx,\by)-f_{D+4d}(\bx,\by)}^2} \le \frac{1}{400},
	\]
	and we will show that this implies the existence of a $\sigma$-network of a similar size and with a similar magnitude of the weights, which gives a similar accuracy when approximating $f_{D+4d}$ uniformly over the discrete domain $\Scal_{4D+16d}$.
	
    Using the law of total expectation, we can break the above expectation into two iterated expectations as follows
	\[
	\Exp_{\bc\sim\Ucal \p{\Ccal_{4D+16d}}}\pcc{ \Exp_{\bz\sim\Ucal(\Scal_{4D+16d})} \pcc{\p{\Ncal(\bz+\bc)-f_{D+4d}(\bz+\bc)}^2}} \le \frac{1}{400}.
	\]
	For $\bc\sim\Ucal \p{\Ccal_{4D+16d}}$, we can define the non-negative random variable
	\[
	X_{\bc}\coloneqq\Exp_{\bz\sim\Ucal(\Scal_{4D+16d})} \pcc{\p{\Ncal(\bz+\bc)-f_{D+4d}(\bz+\bc)}^2}.
	\]
	This allows us to rewrite the former inequality more compactly as
	\begin{equation}\label{eq:markov}
		\Exp_{\bc\sim\Ucal \p{\Ccal_{4D+16d}}}\pcc{X_{\bc}} \le \frac{1}{400}.
	\end{equation}
	Using Markov's inequality on $X_{\bc}$, and by virtue of \eqref{eq:markov}, we have
	\[
	\pr_{\bc\sim\Ucal \p{\Ccal_{4D+16d}}}\pcc{X_{\bc} < \frac{400}{399}\Exp_{\bc\sim\Ucal \p{\Ccal_{4D+16d}}}\pcc{X_{\bc}}\le\frac{1}{399} } \ge 1-\frac{399}{400} > 0.
	\]
	Namely, there must exist some $\bc\in\Ccal_{4D+16d}$ such that
	\begin{equation}\label{eq:399}
		\Exp_{\bz\sim\Ucal(\Scal_{4D+16d})} \pcc{\p{\Ncal(\bz+\bc)-f_{D+4d}(\bz+\bc)}^2} = X_{\bc} \le \frac{1}{399}.
	\end{equation}
	Next, we have by the definition of $f$ that $f_{D+4d}(\bz+\bc) = f_{D+4d}(\bz)$ for all $\bc\in\Ccal_{4D+16d}$ and all $\bz\in\Scal_{4D+16d}$. Moreover, there exists a $\sigma$-network $\tilde{\Ncal}$ of the same depth and width as $\Ncal$, such that $\Ncal(\bz+\bc)=\tilde{\Ncal}\p{\bz}$ for all $\bc\in\Ccal_{4D+16d}$ and all $\bz\in\Scal_{4D+16d}$. This holds true since shifting the input by a constant vector $\bc$ merely shifts the biases in the first hidden layer by the same vector. Since this operation is linear, it can be absorbed in the hidden layer without changing the architecture of $\tilde{\Ncal}$, and where the magnitude of the weights is now at most $4dC$.
	
	The above observations, the definition of $X_{\bc}$, and \eqref{eq:399}, imply the existence of some $\bc\in\Ccal_{4D+16d}$ that satisfies
    \[
		\Exp_{\bx,\by\sim\Ucal(\Scal_{4D+16d})} \pcc{\p{\tilde{\Ncal}(\bz)-f_{D+4d}(\bz)}^2}
		= \Exp_{\bz\sim\Ucal(\Scal_{4D+16d})} \pcc{\p{\Ncal(\bz+\bc)-f_{D+4d}(\bz+\bc)}^2}
		\le \frac{1}{399}.
    \]
	
    Next, letting $\beta=\{1,\alpha_{\sigma}\}$ where $\alpha_{\sigma}$ is the constant guaranteed by \asmref{asm:lb_relaxed}, we assume that the magnitude of the weights of $\tilde{\Ncal}$, $4dC$, is at most $2^{\frac{d}{80\beta+80}}$. This is justified, since otherwise we would have that $C=\Omega(\exp(\Omega(d)))$, in which case the lower bound we are proving becomes trivial for sufficiently small constants hidden in the asymptotic notation. Similarly, we assume that the width of $\tilde{\Ncal}$, $w(D+4d,4dC)$, is at most $2^{\frac{d}{80\beta}}$, since otherwise the lower bound we are proving already holds. For the $\ip_d(\cdot,\cdot)$ function, $c\coloneqq\frac{1}{399}$ and $\varepsilon\coloneqq \frac{1}{398}-c$, we have by (the contrapositive of) \thmref{thm:ball_reduction} that
    \[
        \inf_{\Ncal\in\Fcal_{2d,3dC,\beta}} \Exp_{\bx,\by\sim \Ucal\p{\set{0,1}^{d}}}\pcc{\p{\Ncal(\bx,\by) - \ip_d(\bx,\by)}^2} \le c+\varepsilon = \frac{1}{398}.
    \]
    Namely, there exists some $\Ncal':\reals^{2d}\to\reals$ such that
    \begin{equation}\label{eq:uni_approx2}
        \Exp_{\bx,\by\sim \Ucal\p{\set{0,1}^{d}}}\pcc{\p{\Ncal'(\bx,\by) - \ip_d(\bx,\by)}^2} \le\frac{1}{398}.
    \end{equation}

	Let $n\in\mathbb{N}$ to be determined later, we now construct a neural network $\Ncal'':\reals^{D+4d}\to\reals$ that will achieve a small margin on all the inputs $\bx,\by\in\{0,1\}^d$ simultaneously. The network will consist of $n$ blocks, where each of which has the architecture of $\Ncal'$, and is re-randomized using the following process for every $j\in[n]$:
	\begin{itemize}
		\item
		We sample two binary vectors $\bx_j',\by_j'\in\{0,1\}^d$ uniformly at random.
		
		\item 
		We sample two binary vectors $\bx_j'',\by_j''\in\{0,1\}^D$ uniformly at random, until the number of pairs $(x_{j,i},y_{j,i})$ that are both one is an even number for every $j\in[n]$.
		
		\item 
		We concatenate our sample into two binary vectors $(\bx+\bx',\bx',\bx+\bx',\bx', \bx'')$ and $(\by+\by',\by',\by',\by+\by', \by'')$.
		
		\item 
		We sample a permutation $\tau_j:\reals^{D+4d}\to\reals$ uniformly at random, and use this permutation to permute the previous two binary vectors, denoting the results as $\hat{\bx}_j$ and $\hat{\by}_j$, respectively.
		
		\item 
		We modify each $\Ncal'_j$ to simulate their computation on the inputs $\hat{\bx}_j,\hat{\by}_j$. For the coordinates receiving $\bx+\bx'$ or $\by+\by'$ as input, this can be done by keeping the weights in the hidden layer unchanged in the case where their corresponding coordinate in $\bx'$ was drawn as $0$, and by composing the weights with the transformation $x_i\mapsto 1-x_i$ whenever $x_i'=1$. Since this is a linear transformation of the input, it can be absorbed into the weights of the first hidden layer without changing the architecture of $\Ncal_j'$. Lastly, multiplying the weights of the hidden layer with the permutation matrix which corresponds to the sampled permutation $\tau_j$, which is also a linear transformation, also allows us to keep the architecture of $\Ncal_j'$ unchanged.
	\end{itemize}
	We now turn to bound the approximation error in absolute value of the network $\Ncal'$ when approximating $\ip_{D+4d}$. Fix some $\bx,\by\in\{0,1\}^d$, and let $\Dcal$ denote the distribution over the set $\{0,1\}^{2(D+4d)}$ which is induced by the randomness in picking $\bx'_j,\by'_j\in\{0,1\}^{d}$, $\bx_j'',\by_j''\in\{0,1\}^{D}$ and a uniformly chosen permutation $\tau_j$ of $[D+4d]$, as described in the above random construction of $\Ncal_j'$. Cauchy-Schwarz, combined with \eqref{eq:uni_approx2} and \propref{prop:combi}'s bound on $\norm{\Dcal}_2$ yields
	\begin{align}
		\E_{(\hat{\bx},\hat{\by})\sim\Dcal}\pcc{\left|\Ncal'(\hat{\bx},\hat{\by})-\ip_{D+4d}(\hat{\bx},\hat{\by})\right|} &= \sum_{\hat{\bx},\hat{\by}\in \{0,1\}^{D+4d}} \pr_{\Dcal}\pcc{X=\hat{\bx},Y=\hat{\by}}\cdot \left|\Ncal'(\hat{\bx},\hat{\by})-\ip_{D+4d}(\hat{\bx},\hat{\by})\right|\nonumber\\
		& \le ||\Dcal||_2\sqrt{\sum_{\hat{\bx},\hat{\by}\in \{0,1\}^{D+4d}} \p{\Ncal'(\hat{\bx},\hat{\by})-\ip_{D+4d}(\hat{\bx},\hat{\by})}^2}\nonumber\\
		&\le 8\sqrt{\frac{1}{398}} < 0.41. \label{eq:D_norm_bound}
	\end{align}

	Next, we formally define $\Ncal''$ as the network
	\[
	\Ncal''(\bx,\by)\coloneqq\frac{1}{n}\sum_{j=1}^n\Ncal_j'(\hat{\bx}_j,\hat{\by}_j).
	\]
	Note that since $\Ncal''$ is a linear combination of depth 2 networks, it is in itself a depth 2 network. Moreover, we remark that this random process always preserves the value of the inner product. To see this, fix any $\bx',\by',\bx'',\by'',\tau$ chosen according to the above random process. Then by taking equalities that are $\text{mod } 2$, we have
	\begin{align*}
		&\ip_{D+4d}(\tau(\hat{\bx}),\tau(\hat{\by})) = \ip_{D+4d}(\hat{\bx},\hat{\by})\\
		&\hskip 2cm= \ip_d(\bx+\bx',\by+\by') + \ip_d(\bx',\by') + \ip_d(\bx+\bx',\by') + \ip_d(\bx',\by+\by') + \ip_D(\bx'',\by'')\\
		&\hskip 2cm= \ip_d(\bx+\bx',\by+\by') + \ip_d(\bx',\by') + \ip_d(\bx+\bx',\by') + \ip_d(\bx',\by+\by')\\
		&\hskip 2cm= \ip_d(\bx,\by+\by') + \ip_d(\bx,\by') = \ip_d(\bx,\by).
	\end{align*}
	where the first equality follows from the fact that permutations preserve sums, the second equality is by the definition of our construction of $\hat{\bx},\hat{\by}$, the third equality is due to $\bx''$ and $\by''$ always having an even number of pairs $x_i=1,y_i=1$ which implies $\ip_D(\bx'',\by'')=0$, and the last two equalities follow from basic properties of the inner product mod 2. Thus, we have
	\[
	\ip_{d}(\bx,\by) = \ip_{D+4d}(\hat{\bx}_j,\hat{\by}_j)
	\]
	for all $\bx,\by\in\{0,1\}^d$ and all $j\in[n]$. Having defined $\Ncal''$, we now bound the approximation error in absolute value which it achieves on an arbitrary input $\bx,\by\in\{0,1\}^d$, over the randomness induced by the previously described process. Compute
	\begin{align*}
		\abs{\Ncal''(\bx,\by)-\ip_{d}(\bx,\by)} &= \abs{\frac{1}{n}\sum_{j=1}^n\Ncal_j'(\hat{\bx}_j,\hat{\by}_j) - \E\pcc{\Ncal_1'(\hat{\bx}_1,\hat{\by}_1)} + \E\pcc{\Ncal_1'(\hat{\bx}_1,\hat{\by}_1)} - \ip_{D+4d}(\hat{\bx}_1,\hat{\by}_1)}\\
		&\le \abs{\frac{1}{n}\sum_{j=1}^n\Ncal_j'(\hat{\bx}_j,\hat{\by}_j) - \E\pcc{\Ncal_1'(\hat{\bx}_1,\hat{\by}_1)}} + \E\pcc{\abs{\Ncal_1'(\hat{\bx}_1,\hat{\by}_1) - \ip_{D+4d}(\hat{\bx}_1,\hat{\by}_1)}},
	\end{align*}
	where the expectation is taken over the randomness in sampling $(\hat{\bx}_j,\hat{\by}_j)$ from $\Dcal$. Using \eqref{eq:D_norm_bound}, we can upper bound the above by
	\begin{equation}\label{eq:hoeffding}
		\abs{\frac{1}{n}\sum_{j=1}^n\Ncal_j'(\hat{\bx}_j,\hat{\by}_j) - \E\pcc{\Ncal_1'(\hat{\bx}_1,\hat{\by}_1)}} + 0.41.    
	\end{equation}
	
	Next, we will use Hoeffding's inequality to upper bound the absolute value term above, but first we will derive an upper bound on the magnitude of the output of the network $\Ncal_j'$. We have that each neuron in $\Ncal_j'$ has weights of magnitude at most $2C$, therefore, from \asmref{asm:lb_relaxed} we get that the output of each hidden neuron is upper bounded by $\poly(d,C)$ over the domain $\{0,1\}^{D+4d}$. This implies that the output of the output neuron is at most $B\coloneqq\poly(d,C)w(D+4d,C)$, since the width of $\Ncal_j'$ is $w(D+4d,C)$ and its output neuron has weights bounded by $C$. We now use this bound, and the fact that $\Ncal_j'(\hat{\bx}_j,\hat{\by}_j)$, $j=1,\ldots,n$ are independent with respect to the randomness in sampling $\bx_j',\bx_j'',\by_j',\by_j'',\tau_j$, to invoke Hoeffding's inequality, yielding
	\[
	\pr\pcc{\abs{\frac{1}{n}\sum_{j=1}^n\Ncal_j'(\hat{\bx}_j,\hat{\by}_j) - \E\pcc{\Ncal_1'(\hat{\bx}_1,\hat{\by}_1)}} > 0.04} \le 2\exp\p{-0.5n\frac{0.04^2}{B^2}}.
	\]
	Setting $n=2500B^2d=\poly(d,C,w(D+4d,C))$, we can upper bound the above by $2\exp(-2d)$, which is strictly less than $2^{-2d}$ for all $d\ge2$. Taking a union bound over all $2^{2d}$ possibilities for the inputs $(\bx,\by)\in\{0,1\}^{2d}$, we have by substituting our Hoeffding bound in \eqref{eq:hoeffding}, that with positive probability, $\Ncal''$ satisfies $\abs{\Ncal''(\bx,\by)-\ip_{d}(\bx,\by)} \le0.41+0.04=0.45$ for all inputs $(\bx,\by)\in\{0,1\}^{2d}$. By the probabilistic method, this implies the existence of particular realizations of the random variables 
	\[
	\bx_1',\ldots,\bx_n',\bx_1'',\ldots,\bx_n'',\by_1',\ldots,\by_n',\by_1'',\ldots,\by_n'',\tau_1',\ldots,\tau_n'
	\]
	with this property. Define the neural network $\Ncal'''$ as the network obtained from substituting these variables with the above realizations, and note that doing so merely decreases the input dimension of the network, while keeping the computation in the hidden layer linear, which thus does not change the architecture of $\Ncal'''$ and maintains its width. 
	
	To conclude the derivation so far, we have shown the existence of a depth 2 $\sigma$-network $\Ncal'''$, which has width $\poly(d,C,w(D+4d,C))$, and satisfies
	\[
	\max_{(\bx,\by)\in\{0,1\}^{2d}}\abs{\Ncal'''(\bx,\by) - \ip_d(\bx,\by)} < 0.45.
	\]
	We now use \propref{prop:thresh_approx} with $\delta = 0.04$ to obtain a threshold network $\bar{\Ncal}$ from $\Ncal'''$, having width $\poly(d,C,w(D+4d,C))$ and weights of magnitude at most $\poly(d,C)$, such that for all $\bx,\by\in\{0,1\}^d$
	\[
	\abs{\bar{\Ncal}(\bx,\by) - \ip_d(\bx,\by)} \le \abs{\bar{\Ncal}(\bx,\by) - \Ncal'''(\bx,\by)} + \abs{\Ncal'''(\bx,\by) - \ip_d(\bx,\by)} \le 0.45 + 0.04 = 0.49.
	\]
	Constructing a threshold circuit from $\bar{\Ncal}$, which employs a threshold activation on its output neuron, we obtain a threshold circuit which computes $\ip_d(\cdot,\cdot)$. We lower bound the width of the circuit by using the following fact adapted from \citet{hajnal1993threshold} which appears in \citet{martens2013representational}, and is stated here in a slightly modified manner for the sake of completeness.
	\begin{fact}
		For a depth 2 threshold network $\Ncal$ of width $m$ and weights bounded in magnitude by $C$, that satisfies
		\[
		\max_{(\bx,\by)\in\{0,1\}^{2d}}\abs{\Ncal(\bx,\by) - \ip_d(\bx,\by)} \le 0.5-\delta
		\]
		for some $\delta\in(0,0.5)$, we have
		\[
		m\ge\Omega\p{\frac{\delta 2^{d/3}}{C}}.
		\]
	\end{fact}
	Substituting $\delta=0.01$, the above fact and our expression for the width of $\bar{\Ncal}$ imply the inequality
	\[
	\poly(d,C,w(D+4d,C)) \ge \Omega\p{\frac{2^{d/3}}{\poly(d,C)}}.
	\]
	Simplifying the above by using more asymptotic notation and absorbing terms that are polynomial in $d$ into the exponent, we have
	\[
	w(D+4d,C) \ge \Omega\p{\frac{2^{\Omega\p{d}}}{\poly(C)}}.
	\]
	Recall that $D=100d$. Letting $d'\coloneqq104d$ which implies $d=\Omega(d')$ and performing a change of variables, the theorem follows.
\end{proof}


\subsection{Proof of \thmref{thm:ub}}\label{app:ub}

\begin{proof}[\unskip\nopunct]
    Since the first $2d$ coordinates of our input hold no useful information for computing $f_d$, we simply ignore them. Next, by performing a simple change of variables (see \citet[Thm.~9]{safran2019depth}), we can assume for convenience that the remaining $2d$ coordinates are scaled by a factor of $3\sqrt{d}$ (which will scale the magnitude of the weights), so that the domain of approximation is $[0,1]^{2d}$.
    
	Let $\varepsilon>0$. First, if $\varepsilon>\frac12$, then we have
	\[
	\sup_{\bx,\by\in\Acal_d}\abs{\frac12-f_d(\bx,\by)} = \frac{1}{2}.
	\]
	Namely, a network which computes a constant function satisfies our requirements. We can thus assume from now on that $\varepsilon\le\frac12$.
	
	Next, we define a few auxiliary functions that will be used in the construction of our approximation. Let
	\[
	g_1(z)\coloneqq
	\begin{cases}
		0, & z\in(-\infty,5],\\
		z-5, & z\in(5,6),\\
		1, & z\in[6,\infty),
	\end{cases}
	\]
	and Let
	\[
	g_2(z)\coloneqq
	\begin{cases}
		z\mod1, & \lfloor z\rfloor\text{ even},\\
		1-(z\mod1), & \lfloor z\rfloor\text{ odd}.\\
	\end{cases}
	\]
	Let $\delta_1\coloneqq\frac{\varepsilon}{2d}$, we have from \asmref{asm:ub} and the fact that $g_1(\cdot)$ is $1$-Lipschitz, that there exists some depth 2 $\sigma$-network $h_1$, of width $\Ocal\p{\frac{d}{\varepsilon}}$ and weights bounded by $\Ocal\p{\frac{1}{\varepsilon}}$, which satisfies 
	\begin{equation}\label{eq:h1_approx}
		\abs{g_1(z)-h_1(z)} \le \frac{\varepsilon}{2d},\hskip 0.3cm \forall z\in[0,8].
	\end{equation}
	Likewise, there exists some depth 2 $\sigma$-network $h_2$, of width $\Ocal\p{\frac{d}{\varepsilon}}$ and weights bounded by $\Ocal\p{\frac{1}{\varepsilon}}$, which satisfies
	\begin{equation}\label{eq:h2_approx}
		\abs{g_2(z)-h_2(z)} \le \frac{\varepsilon}{2},\hskip 0.3cm \forall z\in[-2d-1,2d+1].
	\end{equation}
	
	Now, it is easy to verify that
	\[
	g_1(4x+4y) = \AND(\round(x),\round(y))\hskip 0.3cm \forall x,y\in\Acal_1,
	\]
	implying
	\[
	\sum_{i=1}^{d}g_1(4x_i+4y_i) = \inner{\round(\bx),\round(\by)}\hskip 0.3cm \forall \bx,\by\in\Acal_d,
	\]
	and thus
	\[
	g_2\p{\sum_{i=1}^{d}g_1(4x_i+4y_i)} = \ip_d\p{\round(\bx),\round(\by)} = f_d(\bx,\by)\hskip 0.3cm \forall \bx,\by\in\Acal_d.
	\]
	
	We now define the network $\Ncal$ which approximates $f_d(\bx,\by)$ well on the set $\Acal_d$. For all $\bx,\by\in\Acal_d$, define
	\[
	\Ncal(\bx,\by)\coloneqq h_2\p{\sum_{i=1}^{d}h_1(4x_i+4y_i)}.
	\]
	We construct a depth 3 neural network that computes the above function as follows: 
	\begin{itemize}
		\item 
		The first hidden layer will consist of $d$ copies $h_{1,i}$, $i=1,\ldots,d$, of the depth 2, width $\Ocal\p{\frac{d}{\varepsilon}}$ network $h_1$. Each $h_{1,i}$ will receive $(x_i,y_i)$ as input; thus, the weight assigned to the coordinates $x_i,y_i$ is set to $4$, and the weight assigned to the remaining coordinates is set to zero. Note that the width of the first layer is therefore $\Ocal\p{\frac{d^2}{\varepsilon}}$, and the magnitude of its weights is bounded by $\Ocal\p{\frac{\sqrt{d}}{\varepsilon}}$.
		\item 
		The second hidden layer will consist of a single copy of the network computing the function $h_2$. This implies that the width of this layer is $\Ocal\p{\frac{d}{\varepsilon}}$. Each neuron in this layer will assign the same set of weights for each incoming output from the output neurons of the components $h_{1,i}$, $i=1,\ldots,d$. Since this sum is a linear operation performed on these output neurons, we can effectively absorb the output neurons into the hidden neurons in the second layer, and thus avoid adding a third hidden layer. Note that by absorbing weights of magnitude at most $\Ocal\p{\frac{1}{\varepsilon}}$ in a layer with weights bounded by $\Ocal\p{\frac{d}{\varepsilon}}$, we get that the weights in the second hidden layer are of magnitude at most $\Ocal\p{\frac{d}{\varepsilon^2}}$.
	\end{itemize}
	Concluding our construction of $\Ncal$, we have that its width is $\Ocal\p{\frac{d^2}{\varepsilon}}$, due to the first hidden layer; and its magnitude of the weights is $\Ocal\p{\frac{d}{\varepsilon^2}}$, due to the second hidden layer.
	
	Denote $\Bcal_{d}\coloneqq\p{[0,0.25]\cup[0.75,1]}^d$. It is therefore only left, given arbitrary $\bx,\by\in\Bcal_d$, to upper bound the expression
	\begin{align}
		\abs{\Ncal(\bx,\by) - f_d(\bx,\by)} &= \abs{h_2\p{\sum_{i=1}^{d}h_1(4x_i+4y_i)} - g_2\p{\sum_{i=1}^{d}g_1(4x_i+4y_i)}}\nonumber\\
		&\le\abs{h_2\p{\sum_{i=1}^{d}h_1(4x_i+4y_i)} - g_2\p{\sum_{i=1}^{d}h_1(4x_i+4y_i)}} \nonumber\\ &\hskip 1.5cm+\abs{g_2\p{\sum_{i=1}^{d}h_1(4x_i+4y_i)} - g_2\p{\sum_{i=1}^{d}g_1(4x_i+4y_i)}}. \label{eq:ub_approx}
	\end{align}
	We begin with upper bounding the first absolute value term. By virtue of \eqref{eq:h1_approx}, we have
	\begin{equation}\label{eq:sum_h1_minus_g1}
		\abs{\sum_{i=1}^{d}h_1(4x_i+4y_i) - \sum_{i=1}^{d}g_1(4x_i+4y_i)} \le \sum_{i=1}^{d}\abs{h_1(4x_i+4y_i) -g_1(4x_i+4y_i)} \le \frac{\varepsilon}{2},
	\end{equation}
	implying
	\[
	\abs{\sum_{i=1}^{d}h_1(4x_i+4y_i)} \le \abs{\sum_{i=1}^{d}g_1(4x_i+4y_i) } +\frac{\varepsilon}{2} \le \sum_{i=1}^{d}\abs{g_1(4x_i+4y_i) } +\frac14 \le d+\frac14,
	\]
	for all $\bx,\by\in\Bcal_d$. The above and \eqref{eq:h2_approx} imply
	\begin{equation}\label{eq:first_abs_term}
		\abs{h_2\p{\sum_{i=1}^{d}h_1(4x_i+4y_i)} - g_2\p{\sum_{i=1}^{d}h_1(4x_i+4y_i)}} \le \frac{\varepsilon}{2}.
	\end{equation}
	Moving on to bound the second absolute value term in \eqref{eq:ub_approx}, we have by the fact that $g_2(\cdot)$ is $1$-Lipschitz that
	\[
	\abs{g_2\p{\sum_{i=1}^{d}h_1(4x_i+4y_i)} - g_2\p{\sum_{i=1}^{d}g_1(4x_i+4y_i)}} \le \abs{\sum_{i=1}^{d}h_1(4x_i+4y_i) - \sum_{i=1}^{d}g_1(4x_i+4y_i)} \le \frac{\varepsilon}{2},
	\]
	where in the second inequality we used \eqref{eq:sum_h1_minus_g1}. Plugging the above and \eqref{eq:first_abs_term} back in \eqref{eq:ub_approx}, we get
	\[
	\abs{\Ncal(\bx,\by) - f_d(\bx,\by)}\le \varepsilon,
	\]
	for all $\bx,\by\in\Bcal_d$, as desired.
\end{proof}

\end{document}